\documentclass[12pt]{article}
\usepackage{amsmath}
\usepackage{graphicx}
\usepackage{natbib}
\usepackage{url} % not crucial - just used below for the URL 
\usepackage{booktabs}
\usepackage{threeparttable}
\usepackage{subfig} 
\usepackage{float}
\newtheorem{theorem}{Theorem}

\newtheorem{lemma}{Lemma}

\newtheorem{remark}{Remark}

\newtheorem{definition}{Definition}

\newtheorem{proof}{Proof}
\newcommand\argmin{\operatorname{arg\,min}}
\newcommand\argmax{\operatorname{arg\,max}}
\usepackage{amssymb}
\usepackage{bm}
\usepackage{algorithmic}
\usepackage{booktabs}
\usepackage{amsmath}
\usepackage{algorithm}
\usepackage{algorithmic}
    
\usepackage{color}
\usepackage{algorithm}
\usepackage{verbatim}%????
\usepackage{algorithmic}
\usepackage{makecell}
\usepackage{enumerate}
\usepackage{graphicx}
\usepackage{threeparttable}
\usepackage{multirow}
\usepackage{mathrsfs}
\usepackage{caption}
\usepackage{ulem}%删除线
\usepackage{cancel}

\def\bs{\boldsymbol}

\newcommand{\blind}{0}

\begin{document}

	\def\spacingset#1{\renewcommand{\baselinestretch}%
		{#1}\small\normalsize} \spacingset{1}

	%%%%%%%%%%%%%%%%%%%%%%%%%%%%%%%%%%%%%%%%%%%%%%%%%%%%%%%%%%%%%%%%%%%%%%%%%%%%%%
	
	\if0\blind
	{
		\title{\bf Angle-Based Cost-Sensitive Multicategory Classification}
		\author{Yi Yang$^{1}$,  Yuxuan Guo$^{2}$, and Xiangyu Chang$^{1}$\thanks{
				Xiangyu Chang is the corresponding author and the authors gratefully acknowledge \textit{the supporing of National Natural Science Foundation of China (NSFC, 11771012, 61502342 and U1811461).}}\hspace{.2cm}  \\
			$^{1}$Center of Intelligent Decision-Making and Machine Learning\\
			School of Management, Xi'an Jiaotong University, China\\
			$^{2}$ School of Statistics, Renmin University of China, China}
		\maketitle
	} \fi
	
	\if1\blind
	{
		\bigskip
		\bigskip
		\bigskip
		\begin{center}
			{\LARGE\bf Title}
		\end{center}
		\medskip
	} \fi
	
	\bigskip
	\begin{abstract}
		{Many real-world classification problems come with costs which can vary for different types of misclassification}. It is thus important to develop cost-sensitive classifiers which minimize the total misclassification cost. Although binary cost-sensitive classifiers have been well-studied, solving {multicategory classification problems is still challenging.} A popular approach to address this issue is to construct $K$ classification functions for a $K$-class problem and remove the redundancy by imposing a sum-to-zero constraint. {However, such method usually results in higher computational complexity and inefficient algorithms}. In this paper, we propose a novel angle-based cost-sensitive classification framework for multicategory classification without the sum-to-zero constraint. {Loss functions that included in the angle-based cost-sensitive classification framework are further justified to be Fisher consistent. To show the usefulness of the framework, two cost-sensitive multicategory boosting algorithms are derived as concrete instances. Numerical experiments demonstrate that proposed boosting algorithms yield competitive classification performances against other existing boosting  approaches.}
	\end{abstract}
	
	\noindent%
	\vfill
	
	\newpage
	\spacingset{1.5} % DON'T change the spacing!
	\section{Introduction}
	\label{sec:intro}
	
	{In traditional statistical learning techniques, the classification algorithms are designed typically to minimize the misclassification error. This implicitly presumes that all types of misclassification errors have equal costs, which finally leads to cost-insensitive classifiers.} In fact, many real-world classification problems are cost-sensitive, such as fraud detection \citep{Sahin2013, Nami2018}, medical diagnosis \citep{Yang2009, Park2011} and face recognition\citep{Zhang2010F, Zhang2016F}. In these practical applications, the costs of different types of misclassification errors could be vastly different \citep{Sun2007}. {Cost-sensitive learning, unlike the regular cost-insensitive learning, takes the varying costs associated with misclassifying examples into considerations.} It aims at minimizing the total misclassification cost instead of errors, which is of more practical significance. 
	
	In the past twenty years, cost-sensitive learning has attracted much attenetion from researchers. Studies in this field mainly fall into three categories. The first category weights the data space on the basis of translation theorem \citep{Zadrozny2003}. This kind of approaches modifies the distribution of the training set with regards to misclassification cost. So that the distribution of examples is made biased towards the high-cost classes, and a cost-insensitive classifier is then applied. The second class of techniques utilizes the Bayes risk theory to assign each example to the class which has the lowest expected cost \citep{Zhou2005Training, Masnadi-Shirazi2011}. For the last category, the cost-sensitive considerations or stetrgies are directly embeded into a specific classification paradigm. For example, the cost-sensitive node split criteria or pruning schemes are applied in the tree-building process to derive cost-sensitve decision trees \citep{Drummond2000, Elkan2001}. 
	
	Most of these research works, however, have focused on the binary classification problem. In multi-class cases, the construction of cost-sensitive learning approaches is much more complex since misclassifications can occur in more than one way. Several attempts have been made in the previous literature to address the multi-class cost-sensitive classification problem. For example, \cite{Lee2004} indicated the Bayes decision rule along with different multi-class misclassification costs and derived multicategory support vector machine for both cost-insensitive and cost-sensitive classification. By designing a cost-sensitive multiclass exponential loss, \cite{Liu2011} proposed a novel multiclass boosting algorithm in cost-sensitive context. In addition, some efforts have been also devoted to develop cost-sensitive multicategory neural networks and decision trees \citep{Zhou2010On, Zhang2017BP}. Most of these studies address a cost-sensitive multicategory classification problem via a sequence of binary classifiers, or a classification function vector considering all of the classes simultaneously. However, the multiple binary classifiers scheme can be suboptimal in certain situations \citep{Zhang2014, Fu2018Adaptively}. As for the second approach, a sum-to-zero constraint on the function vector is commonly used to reduce the parameter space and ensure the uniqueness of the optimal solution. This usually results in higher computational complexity and cost~\citep{Zhang2014, Zhang2016Reinforced, Fu2018Adaptively}.
	
	{To overcome the disadvantages of the existing multi-class cost-sensitive classifiers mentioned above, this paper proposes a new angle-based cost-sensitive multicategory classification framework.} Using the simplex coding to construct a ($K-1$)-dimensional decision function vector for $K$-class problems under the angle-based framework \citep{Zhang2014, Zhang2016Reinforced, Fu2018Adaptively}, the proposed classification method treats all classes in a simultaneous fashion without the sum-to-zero constraint. Thus, the computational complexity can be highly reduced. {To this end, we first extend the notion of Fisher-consistency defined in \cite{Lin2004} to cost-sensitive multicategory classification problems using angle-based formulations. Then, we propose a family of angle-based loss functions that are justified to be Fisher-consistent for cost-sensitive multicategory learning. To demonstrate the usefulness and effectivness of the proposed framework, two new cost-sensitive multicategory boosting algorithms are derived as concret examples. We verify their performance by comparing them with previous multiclass boosting algorithms in both simulated and real-data experiments. The results show that the proposed methods yield competitive performance compared with other boosting algorithms in both cost-insensitive and cost-sensitive scenarios.
	}
	%Our methods have several attractive properties.
	
	{The rest of the paper is organized as follows. Section \ref{Sec2} presents a brief review of cost-sensitive learning and the angle-based classification framework. The reason why the existing angle-based multicategory classification framework could not be generalized directly to the cost-sensitive version is also discussed as well. In Section \ref{Sec3}, we define the Fisher consistency of angle-based loss functions for cost-sensitive multicategory classification. A family of angle-based loss functions which are Fisher-consistent are then proposed. Section \ref{Sec4} describes two novel cost-sensitive multicategory boosting algorithms based on the proposed loss functions. In Section \ref{Sec5}, we present the results and analysis of the experimental study on synthetic and real datasets. The conclusions are given in Section \ref{Sec6}. 
	}
	
	\section{Review of Cost-Sensitive Classification and Angle-Based Multicategory Loss}\label{Sec2}
	
	{
		In this section, we recall the fundamentals of cost-sensitive learning and angle-based multicategory loss functions. Then we show the reason why the existing angle-based multicategory losses could not be generalized directly to the cost-sensitive version. This motivates us to propose a family of novel loss functions to figure out this problem.}
	
	\subsection{Cost-Sensitive Classification}

	{Let $\mathcal D=\{(\bs{x}_i,y_i)\}_{i=1}^n$ be a training set with $n$ independent and identically distributed observations. Here, $\bs{x}_i\in\mathcal X\subseteq\mathbb{R}^d$ denotes the feature vector of the $i$th observation and $y_i\in\mathcal Y=\{1, ..., K\}$ is its corresponding class label.} $K$ is the number of classes and $K>2$ for multicategory problems. We assume that all observations are sampled from an unknown distribution $\rho(\bs{x},y)$, and $(\bs{X}, Y)$ denotes a random pair that is sampled from $\rho(\bs{x},y)$. 
	
	In cost-sensitive classification, a misclassification cost matrix $\bs C$ = $(C_{j,k})\in\mathbb{R}^{K^2}$ is also prespecified, with $C_{j,k}\geq 0$ being the cost of predicting an instance as class $k$ when the true class is $j$. {Typically, we have $C_{j,j}=0$ since there is no cost for correct classification, and $\bs C$ is not necessarily symmetric due to the fact that $C_{j,k}\neq C_{k,j}$ $(j\neq k)$ is quite common in real applications \citep{Wang2013}}. In particular, $C_{j,k}$ is equal to $\mathbb{I}(j\neq k)$ in cost-insensitive case, where $\mathbb{I}(\cdot)$ is an indicator function that has value $1$ in case its argument is true and $0$ otherwise.
	%Note that if $C_{j,k}\neq C_{k,j}$ when $j\neq k$, the classifier is cost-sensitive. If $\bs{C}$ is symmetric, it is cost-insensitive. 
	
	\begin{comment}
	Cost-sensitive classification is motivated by the fact that the costs of different types of misclassification errors {\color{blue} are usually unequal} in real life. {\color{blue}Take the loan application approval as an example, the cost of rejecting the loan to a good customer who is capable of repaying is much smaller than the cost of granting the loan to a defaulting customer} \citep{Gonzalez2013,Lessmann2015}. Medical diagnosis is another commonly seen cost-sensitive application in practice. In this case, {\color{blue}overlooking a patient who is critically ill may lead to the loss of a life due to a delay in treatment}. {\color{blue}Thus, the cost of misclassifying a critically ill target as a mild case is much higher than the cost of false alarm.} Normally, the value of $C_{j,k}$ with ${j\neq k}$ is specified from domain knowledge and it is not necessarily monetary. As mentioned above, a waste of time or the severity of an illness can also be a cost \citep{Elkan2001}. \end{comment}
	
	For a given cost matrix $\bs{C}$, the primary goal in cost-sensitive learning is to construct a classification decision rule $\phi(\bs{x}): \mathcal{X}\rightarrow\mathcal Y$ so that the expected misclassification cost $\mathbb{E}_{Y|\bs{X}}\big[C_{Y, \phi(\bs{x})} |\bs{X} =\bs{x}\big]$, instead of the expected misclassification error rate in cost-insensitive case, could be minimized. 
	
	\subsection{Angle-Based Multicategory Loss}
	
	For a $K$-category classification problem with $K\geq2$, a regular approach in the literature is to map $\bs{x}$ to a classification function vector $\bs{f}(\bs{x})=(f_1(\bs{x}),\dots,f_K(\bs{x}))^\top\in\mathbb{R}^{K}$. The {\it max rule}, $\argmax_{k} f _{k}(\bs{x})$ where $f_{k}(\bs{x})$ is the $k$th element of $\bs{f}$, then is applied for class label assignment. Typically, a sum-to-zero constraint on $\bs{f}$, {i.e. $\sum_{k=1}^Kf_k(\bs{x})=0$}, is also imposed to ensure the uniqueness of the optimal solution as well as to achieve desirable {statistical} properties \citep{Zhu2009,Zhang2013,FERNANDEZBALDERA2014,FERNANDEZBALDERA2018}. However, constructing $K$ classification functions simultaneously and removing the redundancy by the sum-to-zero constraint not only result in high computational cost, but also introduce extra variability in the estimated classifier \citep{Zhang2014}. {The angle-based method~\citep{Zhang2014} is proposed for overcoming these difficulties. By representing the multicategory class label based on a simplex structure in $\mathbb{R}^{K-1}$, the angle-based method develops the new functional margins which implicitly satisfy the sum-to-zero constraint for classifier construction. As a result, the complexity of the corresponding optimization could be significantly reduced. \cite{Zhang2014} proposed a general framework of the multicategory angle-based classification technique. After that, some extensions and applications of this method have been carried out \citep{Zhang2016Reinforced,Fu2018Adaptively,Liu2018smac}. }
	
	To develop angle-based classifiers for multicategory classification, as studied in \cite{Zhang2014}, a specific simplex in $\mathbb{R}^{K-1}$ is constructed first. The simplex is defined as a $K$-regular polyhedron in $\mathbb{R}^{K-1}$, whose the $j$-th vertice is formed by
	\begin{equation}\label{W_j_def}
	\bs{w}_{j}=\left\{\begin{array}{ll}{(K-1)^{-1/2}\bs{1}}&{\quad j=1,}\\
	{-\dfrac{1+K^{1/2}}{(K-1)^{3/2}}\bs{1}+\left(\dfrac{K}{K-1}\right)^{1/2}\bs{e}_{j-1}}&{\quad 2\leqslant j\leqslant K,}\end{array}\right.
	\end{equation}
	where $\bs{1}\in\mathbb{R}^{K-1}$ is a vector of $1$, and $\bs{e}_{j}\in\mathbb{R}^{K-1}$ is a vector whose every element is $0$ except the $j$th is $1$. It is obvious that the simplex formed by $\bs{W} =\{\bs{w}_{1}, ..., \bs{w}_{K}\}$ has the center at the origin, and each $\bs{w}_{j}$ has norm 1. Moreover, the angles between any two vectors from $\bs{W}$ are equal. {In this setting}, any ($K-1$)-dimensional vector defines $K$ angles in $[0, \pi]$ with respect to $\{\bs{w}_{1}, ..., \bs{w}_{K}\}$. %Thus, any ($K-1$)-dimensional vector defines $K$ angles in $[0, \pi]$ with respect to $\{\bs{w}_{1}, ..., \bs{w}_{K}\}$. 
	
	Using $\bs{w}_{j}$ to represent the class $j$, an angle-based classifier maps $\bs{x}$ to ${\bs{f}}(\bs{x})\in\mathbb{R}^{K-1}$ and the label prediction for $\bs{x}$ is $\hat{y} = \argmin_{j}\angle({\bs{f}(\bs{x})}, \bs{w}_{j})$, where $\angle(\cdot, \cdot)$ donates the angle between two vectors. In other words, an example is predicted to be the class whose corresponding angle is the smallest. %Based on this prediction rule, the $\mathbb{R}^{K-1}$ space could be splitted into $K$ disjoint classification regions. Figure shows the classification regions for the angle-based method when $K=3$. 
	For a given $\bs{f}(\bs{x})$, the smaller the $\angle(\bs{f}(\bs{x}), \bs{w}_{j})$, the larger the projection of $\bs{f}$ on $\bs{w_{j}}$. %, which leads to $\argmin_{j}\angle(\bs{w}_{j}, \hat{\bs{f}})=\argmax_{j}\langle\bs{w}_{j}, \hat{\bs{f}}\rangle$. 
	Hence, the {\it least-angle rule} is equivalent to $\hat{y} = \argmax_{j}\langle \bs{f}(\bs{x}), \bs{w}_{j}\rangle$, where $\langle \cdot, \cdot\rangle$ is the inner product of two vectors. Then for a given binary large-margin classification loss function $\ell(\cdot)$, an angle-based cost-insensitive classifier could be derived by the following framework, 
	\begin{equation}\label{AB_ERM}
	\min_{\bs{f}\in \mathcal{F}}\Big\{ \dfrac{1}{n} \sum_{i=1}^{n} \ell\big(\langle \bs{f}\left(\bs{x}_{i}\right), \bs{w}_{y_{i} }\rangle\big)+\lambda N(\bs{f})\Big\},
	\end{equation}
	where $\mathcal{F}$ is a hypothesis class of functions, $\lambda>0$ is the regularization parameter and $N(\cdot)$ is the regularizer used to avoid over-fitting. Notice that $\sum_{j=1}^{K}\left\langle \bs{f}(\bs{x}), \bs{w}_{j}\right\rangle=0$ for any $\bs{x}$. Thus, the sum-to-zero constraint is implicitly satisfied by the angle-based method and {the optimization problem can be solved more efficiently than other traditional methods.}
	%it can be optimized more efficiently than other traditional methods. 这里的it指代不清，所以我把这句改了
	
	In fact, the above framework has mainly focused on the cost-insensitive situation. In the optimization formulation (\ref{AB_ERM}),
	%mainly the cost-insensitive situation has been considered in \cite{Zhang2014}, 
	$\ell\big(\langle \bs{f}\left(\bs{x}\right), \bs{w}_{y}\rangle\big)$ measuring the loss of assigning the label $y$ to $\bs{x}$ assumes by default that the penalty of all types of misclassification errors are equal. However, this angle-based multicategory loss function could not be directly generalized to a cost-sensitive version. Let us show this by considering a binary classification {problem} first. %a binary classification case first.
	
	When $K=2$, we have $w_{1}=1$ and $w_{2}=-1$ according to (\ref{W_j_def}), and $\ell\big(\left\langle f(\boldsymbol{x}), w_{y}\right\rangle\big)$ turns into $\ell(y'f(\bs{x}))$ with $y'\in\{-1, 1\}$ accordingly.
	Hence, the angel-based framework (\ref{AB_ERM}) is identical to the regular margin-based one, and it could be easily generalized to a cost-sensitive version by weighting the loss directly according to the cost matrix $\bs{C}$ {{\citep{Bach2006}}} as %这里要加下引用说这个framework的出处吗？其实下一段的最后一句说了这个加权方式有被用于哪些文献。但是这些文献大多数是针对特定loss的扩展，只有Bach的文献明确提出了这个通用的cost-sensitive ERM.
	\begin{equation}\label{frame_CSB}
	\min_{\bs{f}\in \mathcal{F}} \Big\{\dfrac{1}{n} \sum_{i=1}^{n} C^{*}_{y'_{i}, -y'_{i}}\ell(y'_{i}f(\bs{x}_{i}))+\lambda N(f)\Big\},
	\end{equation}
	where $C^{*}_{1, -1}=C_{1,2}$ and $C^{*}_{-1, 1}=C_{2,1}$. For a given example $\bs{x}_{i}$ and its corresponding class $y'_{i}$, if $\bs{x}_{i}$ is misclassified, the weighted loss function $C^{*}_{y'_{i}, -y'_{i}}\ell(y'_{i}f(\bs{x}_{i}))$ will amplify its punishment through the higher value of $\ell$ as well as the cost weight related to the error. Because the true class label $y'_{i}$ is given and there are only two classes, the type of the corresponding misclassification error is definite, and thus the weight could be set directly according to the cost matrix $\bs{C}$. {This sample-based weighting strategy, which makes the loss function cost-sensitive, has been widely adopted in the cost-sensitive binary classification problems \citep{Ting2000,Bach2006,Sun2007,Gu2017}. 
	}
	However, this weighting strategy could not be directly generalized into the multicategory case. {Consider a misclassified example $\bs{x}_{i}$ in multi-class problem now. Because $\ell\big(\langle \bs{f}\left(\bs{x}_{i}\right), \bs{w}_{y_{i} }\rangle\big)$ only reflects the level of inconsistency between the prediction vector and the true class, no information is available regarding which class the $\bs{f}$ actually predicts. For this reason, even though $ y_{i}$ is known, it is still difficult to deduce which type of the misclassification error has been made since it could occur in more than one way.} %it shows no information regarding which class the $\bs{f}$ actually predicts. 
	%However, this weighting strategy could not be directly generalized into the multicategory case. Considering a misclassified example $\bs{x}_{i}$ in multi-class problem, even though $ y_{i}$ is known, it is still difficult to deduce which type of the misclassification error has been made {\color{blue}since it could occur in more than one way. Because $\ell\big(\langle \bs{f}\left(\bs{x}_{i}\right), \bs{w}_{y_{i} }\rangle\big)$ only reflects the level of inconsistency between the prediction vector and the true class, no information is available regarding which class the $\bs{f}$ actually predicts. }%it shows no information regarding which class the $\bs{f}$ actually predicts. 
	As a result, it is unable to determine which entry of $\bs{C}$ should be applied as the corresponding cost weight for loss functions. {We overcome this hurdle by proposing novel angel-based cost-sensitive loss functions for multi-class classification in the next section.} 
	%To solve this problem, we propose an angel-based cost-sensitive loss for multi-class classification to overcome this hurdle in the next section.
	%Yi_0923: 第一部分我把原因的顺序调整了下是否能更利于接受？第二部分To solve this problem和to overcome this hurdle重复了，我就把句子整个改了。
	
	\section{Angel-Based Cost-Sensitive Loss with Fisher Consistency}\label{Sec3}
	
	In this section, we first define the Fisher consistency for cost-sensitive multicategory classification in the angel-based {framework}. Then we develop a general form of angel-based cost-sensitive multicategory loss functions which are Fisher-consistent, and some of its statistical properties are derived.
	
	\subsection{Fisher Consistency of Angel-Based Cost-Sensitive Loss}
	
	Fisher consistency, also known as classification calibration \citep{Bartlett2006}, is regarded as one of the most desirable properties of a loss function and a necessary condition for a loss to achieve reasonable performance in classification~\citep{Mannor2002,Lin2004,Bartlett2006, Masnadi-Shirazi2011}. \cite{Lin2004} motivated the concept of Fisher consistency for binary classification problem. He showed that a Fisher-consistent loss can be used to produce a binary margin-based classifier. 
	
	In cost-insensitive binary classification with $y\in\{-1,1\}$, a loss function $\ell$ is Fisher-consistent if and only if the minimizer of $E_{Y|\bs{X}}[\ell(f(\bs{X}),Y)|\bs{X}=\bs{x}]$ has the same sign as ${P_{Y|\bs{X}}(1|\bs{x})}-\frac{1}{2}$ for any $\bs{x}\in\mathcal{X}$ \citep{Lin2004}. In other words, Fisher consistency requires the population minimizer of a loss function to implement the Bayes optimal decision rule of classification. \cite{Zou2008} further generalized this definition to the multicategory situation. They indicated that a loss function $\ell$ is said to be Fisher-consistent for $K$-class classification if for any $\bs{x}\in\mathcal{X}$, the following optimization problem
	$$\hat{\bs{f}}(\bs{x})=\argmin _{\bs{f}} \mathbb{E}_{Y|\bs{X}}\left[\ell(f_{Y}(\bs{X}))|\bs{X}=\bs{x}\right] \quad \text {subject to} \sum_{j=1}^{K} f_{j}(\bs{x})=0$$ has a unique solution $\hat{\bs{f}}$, and $\argmax _{j} \hat{f}_{j}(\bs{x})=\argmax_{j} P_{j}(\bs{x})$ with $P_{j}(\bs{x})= P(Y = j|\bs{X} =\bs{x})$ for $j=1, ..., K$. That is to say, $\hat{\bs{f}}$ should assign an instance $\bs{x}$ to the class with the largest conditional probability.
	
	However, in cost-sensitive case, the Bayes decision boundary is related to the cost matrix $\bs{C}$. \cite{Masnadi-Shirazi2011} discussed the Bayes optimal decision rule for cost-sensitive binary classification problem with $y\in\{-1, 1\}$, which is given by $\text{sign}\left[{P(Y=1|\bs{X}=\bs{x})}- \frac{C_{-1, 1}}{C_{-1, 1}+C_{1, -1}}\right]$. For $K>2$, \cite{Lee2004} showed that the Bayes rule in the cost-sensitive multiclass classification is given by
	\begin{equation}\label{Bayes_CS_M}
	\phi_{B}(\bs{x})=\mathop{\arg\min}_{k}\sum_{j=1}^K C_{j,k}P_{j}(\bs{x}).
	\end{equation}
	%where $P_{j}(\bs{x})= P(Y = j|\bs{X} =\bs{x})$ for $j=1, ..., K$. 
	When $C_{j,k}$  is equal to $\mathbb{I}(j\neq k)$, the cost-sensitive Bayes decision rule $\phi_{B}$ reduces to the standard Bayes rule. 
	
	{On the basis of} (\ref{Bayes_CS_M}), we define the multicategory angle-based Fisher-consistent loss function for cost-sensitive learning as follows.
	\begin{definition}\label{def_BDR_MCSAB}
		An angle-based loss function $\ell_{c}(\cdot)$ is said to be Fisher-consistent for $K$-class classification in cost-sensitive learning if for any ${\bs x}\in\mathcal{X}$,
		the vector $\bs f^{*}$ minimizing $\mathbb{E}_{Y|\bs{X}}[\ell_{c}(\bs f(\bs{X}), Y)|\bs{X} =\bs{x}]$ satisfies that $$\mathop{\arg\max}_{k} \left<\bs f^{*}(\bs {x}), \bs{w}_{k}\right>=\mathop{\arg\min}_{k}\sum_{j=1}^K C_{j,k}P_{j}(\bs{x})$$ and such an argument is unique, where $\bs{w}_{k}$ {$(k=1,\dots,K)$} is denoted in \eqref{W_j_def}.
	\end{definition}
	%Yi_0923:感觉公式l_{c}里的自变量Y是不是已经涵盖了w_{Y}？
	
	Obviously, Definition \ref{def_BDR_MCSAB} is a natural  generalization of Fisher-consistent concept for multicategory classification in the angle-based context. A family of angle-based loss functions with Fisher-consistent property will be further proposed in the following for demonstrating its usefulness.

	\subsection{Angel-Based Cost-Sensitive Loss} \label{ABCSL}
	In this subsection, we characterize a family of angle-based loss functions that are Fisher-consistent for cost-sensitive multicategory learning. They have the form 
	%we show that there are a number of multicategory angle-based Fisher-consistent loss functions for cost-sensitive learning.  
	\begin{equation}\label{Key_loss}
	\ell_{c}(\bs{f}(\bs{x}), y)=\sum_{t=1}^K C_{y,t} \ell(-\left<\bs f(\bs{x}), \bs{w}_{t}\right>),
	\end{equation}
	where $\ell(\cdot)$ could be many large-margin loss functions as long as they satisfy certain conditions. On the basis of (\ref{Key_loss}), an angle-based cost-sensitive classifier for multiclass problem then could be derived from 
	\begin{equation}\label{ERM_our_loss}
	\min_{\bs{f}\in \mathcal{F}} \dfrac{1}{n} \sum_{i=1}^{n} \sum_{t=1}^K C_{y_{i},t} \ell(-\left<\bs f(\bs{x}_{i}),\bs{w}_{t}\right>).
	\end{equation}
	
	Let us consider the proposed loss function (\ref{Key_loss}). In $\ell_{c}(\bs{f}(\bs{x}), y)$, $C_{y,t} \ell(-\left<\bs f(\bs{x}),\bs{w}_{t}\right>)$ is a hybrid of loss value $\ell(-\left<\bs f(\bs{x}),\bs{w}_{t}\right>)$ and misclassification cost $C_{y,t}$, in which $\ell(-\left<\bs f(\bs{x}),\bs{w}_{t}\right>)$ will impose a great penalty on the large value of $\langle \bs{f}\left(\bs{x}\right), \bs{w}_{t}\rangle$ for any $t$ and the cost weight $C_{y,t}$ adjusts this punishment according to the error type. % $\ell_{c}(\bs{f}(\bs{x}), y)$ will penalize large values of all $\langle \bs{f}\left(\bs{x}\right), \bs{w}_{t}\rangle$ when $t\neq y$, and
	Therefore, $\ell_{c}(\bs{f}(\bs{x}), y)$ will encourage a large value of $\langle \bs{f}\left(\bs{x}\right), \bs{w}_{y}\rangle$, due to the fact that $C_{y,y}$ is equal to $0$ and $\sum_{t=1}^{K}\left\langle \bs{f}, \bs{w}_{t}\right\rangle=0$. For a given ($\bs{x}, y$), $\ell_{c}(\bs{f}(\bs{x}), y)$ defined in (\ref{Key_loss}) calculates the weighted loss value over all $\bs{w}_{t}$ and then sums them up. Hence, it not only measures the level of inconsistency between the prediction and the true class, but also takes every type of misclassification error that might occur into considerations. As a result, the corresponding cost weights then could be set directly according to $\bs{C}$, which are just similar to the sample-based weighting strategy in binary classification.
	
	Afterwards, we show through the following theorem the sufficient conditions for (\ref{Key_loss}) to be Fisher-consistent.
	
	\begin{theorem}\label{Th_1}
		The angle-based cost-sensitive %classification 
		loss function $\sum_{t=1}^K C_{y,t} \ell(-\left<\bs f(\bs{x}), \bs{w}_{t}\right>)$ is Fisher consistent if $\ell(z)$ is convex in $z$, the derivative $\ell'(z)$ exists and $\ell'(z)<0$ for all $z$.
	\end{theorem}
	
	\begin{proof}
		According to Definition \ref{def_BDR_MCSAB}, Fisher consistency requires that for a given example $\bs{x}$ such that $\sum_{j=1}^K C_{j,k}P_{j}(\bs{x})<\sum_{j=1}^K C_{j,k'}P_{j}(\bs{x})$ for any  $k'\in \{1,2,...,K\}$ with $k'\neq k$, the $\bs f^{*}$ minimizing $\mathbb{E}_{Y|\bs{X}}[\sum_{t=1}^K C_{Y,t} \ell(-\left<\bs f(\bs{X}),\bs{w}_{t}\right>)|\bs{X} =\bs{x}]$ satisfies  $k=\mathop{\arg\max}_{j} \left<\bs f^{*}(\bs {x}),\bs{w}_{j}\right>$ and such an argument is unique under the angle-based prediction rule.
		%Under the angle-based prediction rule, Fisher consistency requires that for a given $\bs{x}$ such that $\sum_{j=1}^K C_{j,k}P_{j}(\bs{x})<\sum_{j=1}^K C_{j,k'}P_{j}(\bs{x})$ for any $k'\neq k$ and $k'\in \{1,2,...,K\}$, the $\bs f^{*}$ minimizing $\mathbb{E}_{Y|\bs{X}}[\sum_{t=1}^K C_{Y,t} \ell(-\left<\bs f(\bs{X}),\bs{w}_{t}\right>)|\bs{X} =\bs{x}]$ satisfies that $k=\mathop{\arg\max}_{j} \left<\bs f^{*}(\bs {x}),\bs{w}_{j}\right>$ and such an argmax is unique according to Definition \ref{def_BDR_MCSAB}.

		Recall that the definition of $\bs{f}^{*}$ is 
		\begin{equation}\nonumber
		\begin{aligned}
		\bs{f}^{*}&=\mathop{\arg\min}_{\bs{f}}\mathbb{E}_{Y|\bs{X}}\left[\sum_{t=1}^K C_{Y,t} \ell(-\left<\bs f(\bs{X}),\bs{w}_{t}\right>)|\bs{X} =\bs{x}\right]\\
		&=\mathop{\arg\min}_{\bs{f}}\sum_{j=1}^K\sum_{t=1}^K C_{j,t}P_{j} \ell(-\left<\bs{w}_{t}, \bs{f}\right>).
		\end{aligned}
		\end{equation}
		
		{Without} loss of generality, we need to show that when $\sum_{j=1}^K C_{j,1}P_{j}<\sum_{j=1}^K C_{j,2}P_{j}$, then  $\left<\bs{w}_{1},\bs{f}^{*}\right>>\left<\bs{w}_{2},\bs{f}^{*}\right>$. We {argue} this by contradiction.
		
		If $\left<\bs{w}_{1},\bs{f}^{*}\right>\leq\left<\bs{w}_{2},\bs{f}^{*}\right>$, let $\bs{f}^{**}$ be such a vector that satisfies  $\left<\bs{w}_{t},\bs{f}^{**}\right>=\left<\bs{w}_{t},\bs{f}^{*}\right>$ for $t\geq3$ and $\left<\bs{w}_{1},\bs{f}^{**}\right>=\left<\bs{w}_{1},\bs{f}^{*}\right>+\varepsilon$, $\left<\bs{w}_{2},\bs{f}^{**}\right>=\left<\bs{w}_{2},\bs{f}^{*}\right>-\varepsilon$, where $\varepsilon>0$ is a small number. Such a vector $\bs{f}^{**}$ always exists by setting $u=1$, $v=2$, $\bs{f}^{**}=\bs{f}^{*}+z(\bs{w}_{1}-\bs{w}_{2})$ for some $z\in \mathbb{R^{+}}$ in Lemma 1 of \cite{Zhang2014} (P.S. This Lemma is provided in the Appendix for completeness). Then we have
		\begin{equation}\nonumber
		\begin{aligned}
		&\sum_{j=1}^K\sum_{t=1}^K C_{j,t}P_{j} \ell(-\left<\bs{w}_{t},\bs{f}^{**}\right>)-\sum_{j=1}^K\sum_{t=1}^K C_{j,t}P_{j} \ell(-\left<\bs{w}_{t},\bs{f}^{*}\right>)\\
		=&\sum_{j=1}^K P_{j}\left[\sum_{t=1}^K C_{j,t} \ell(-\left<\bs{w}_{t},\bs{f}^{**}\right>)\right]-\sum_{j=1}^K P_{j}\left[\sum_{t=1}^K C_{j,t} \ell(-\left<\bs{w}_{t},\bs{f}^{*}\right>)\right]\\
		=&\sum_{j=1}^K P_{j}\left[\sum_{t=1}^K C_{j,t} \ell(-\left<\bs{w}_{t},\bs{f}^{**}\right>)-\sum_{t=1}^K C_{j,t} \ell(-\left<\bs{w}_{t},\bs{f}^{*}\right>)\right]\\
		=&\sum_{j=1}^K P_{j}\left[C_{j,1}\ell(-\left<\bs{w}_{1},\bs{f}^{**}\right>)+C_{j,2}\ell(-\left<\bs{w}_{2},\bs{f}^{**}\right>)+\sum_{t=3}^K C_{j,t} \ell(-\left<\bs{w}_{t},\bs{f}^{**}\right>)\right.\\
		&{\color{white}tabtab}\left.-C_{j,1}\ell(-\left<\bs{w}_{1},\bs{f}^{*}\right>)-C_{j,2}\ell(-\left<\bs{w}_{2}, \bs{f}^{*}\right>)-\sum_{t=3}^K C_{j,t} \ell(-\left<\bs{w}_{t}, \bs{f}^{*}\right>)\right]\\
		=&\sum_{j=1}^K P_{j}\Big[-C_{j,1}\varepsilon\ell'(-\left<\bs{w}_{1},\bs{f}^{*}\right>)+C_{j,2}\varepsilon\ell'(-\left<\bs{w}_{2},\bs{f}^{*}\right>)\Big]+\textit{O}(\varepsilon)\\
		\leq&\varepsilon \ell'(-\left<\bs{w}_{2},\bs{f}^{*}\right>)\left(\sum_{j=1}^KP_{j}C_{j,2}-\sum_{j=1}^KP_{j}C_{j,1}\right)+\textit{O}(\varepsilon).
		\end{aligned}
		\end{equation}
		The last inequality holds due to the convexity of $\ell(\cdot)$ and the assumption that $\left<\bs{w}_{1},\bs{f}^{*}\right>\leq\left<\bs{w}_{2},\bs{f}^{*}\right>$. Because $\varepsilon>0$, $\ell'(z)<0$ for all $z$ and $\sum_{j=1}^K C_{j,1}P_{j}<\sum_{j=1}^K C_{j,2}P_{j}$, we have $$\sum_{j=1}^K\sum_{t=1}^K C_{j,t}P_{j} \ell(-\left<\bs{w}_{t},\bs{f}^{**}\right>)<\sum_{j=1}^K\sum_{t=1}^K C_{j,t}P_{j} \ell(-\left<\bs{w}_{t},\bs{f}^{*}\right>),$$ and it is in contradiction to the definition of $\bs{f}^{*}$. The desired results then follow.
		%Yi_0924:我整个证明过程写的和Biometrika YUFENG LIU的证明太像了
	\end{proof}
	
	In practice, after the optimal classifier is obtained, the estimation of the expected cost of each class for a given observation is also of great significance. {In the following}, we show the relationship between the theoretical minimizer $\bs{f}^{*}$ and the expected cost of a specific predicted class in Theorem \ref{Th_2}. It is remarkable because it also provides us an approach to estimate the conditional class probabilities in cost-insensitive classification without using the likelihood approach.

	\begin{theorem} \label{Th_2}
		{Under the angle-based classification framework}, suppose the function $\ell$ is differentiable and 
		\begin{equation}\nonumber
		\begin{aligned}
		\bs{f}^{*}&=\mathop{\arg\min}_{\bs{f}}\mathbb{E}_{Y|\bs{X}}\left[\sum_{t=1}^K C_{Y,t} \ell(-\left<\bs f(\bs{X}),\bs{w}_{t}\right>)|\bs{X} =\bs{x}\right],
		\end{aligned}
		\end{equation}
		then the expected cost for the predicted class $t$ can be expressed as $$\mathbb{E}_{Y|\bs{X}}\big[C_{Y,t}|\bs{X} =\bs{x}\big]=\sum_{j=1}^K C_{j,t}P_{j}=\dfrac{-M}{\ell'(-\left<\bs{w}_{t}, \bs{f}^{*}\right>)}.$$ We further assume that $\bs{C}$ is invertible, then the class conditional probability vector {$\bs{p}=(P_{1},..., P_{K})^\top$} can be expressed as    $$\bs{p}=-M(\bs{C}^{T})^{-1}\bs{\ell}^{*},$$
		where $\bs{\ell}^{*}\in\mathbb{R}^{K}$ is a vector whose the $k$th element is $\ell'(-\left<\bs{w}_{k},\bs{f}^{*}\right>)^{-1}$, and $M$ is a normalizing constant. Specifically, in the cost-insensitive case (i.e. $C_{j,t}=\mathbb{I}(j\neq t)$), the class probabilities can be expressed as
		$$P_{t}=1+\dfrac{(1-K)\ell'(-\left<\bs{w}_{t}, \bs{f}^{*}\right>)^{-1}}{\sum_{k=1}^K \ell'(-\left<\bs{w}_{k}, \bs{f}^{*}\right>)^{-1}}$$ for $t=1, 2,..., K$.
	\end{theorem}
	
	\begin{proof}
		Given the class conditional probability vector $\bs{p}=(P_{1},..., P_{K})^{T}$, the expectation of the proposed {cost-sensitive} loss on $\bs{f}^{*}$ is 
		\begin{equation}\label{E_f_star}
		\mathbb{E}_{Y|\bs{X}}\left[\sum_{t=1}^K C_{Y,t} \ell(-\left<\bs f^{*}(\bs{X}),\bs{w}_{t}\right>)|\bs{X} =\bs{x}\right]=\sum_{j=1}^K\sum_{t=1}^K C_{j,t}P_{j} \ell\left(-\left<\bs{w}_{t},\bs{f}^{*}(\bs{x})\right>\right).
		\end{equation}
		{We take partial derivative of (\ref{E_f_star}) with respect to the $r$th element of $\bs{f}^{*}$. For $r=1, 2,..., K-1$, then we have}
		\begin{equation}\label{PD_EF}
		-\sum_{j=1}^K\sum_{t=1}^K C_{j,t}P_{j} \ell'(-\left<\bs{w}_{t},\bs{f}^{*}\right>)w_{t,r}=0,
		\end{equation}
		where $w_{t,r}$ is the $r$th element of $\bs{w}_{t}$. {Notice (\ref{PD_EF}) can be reformulated as follows:}
		\begin{equation}\label{derivative_eq}
		\sum_{t=1}^K \left[\sum_{j=1}^K -C_{j,t}P_{j} \ell'(-\left<\bs{w}_{t},\bs{f}^{*}\right>)\right]\bs{w}_{t}=\bs{0}_{K-1},
		\end{equation}
		where $\bs{0}_{K-1}$ is a vector with length $K-1$ and each element $0$. {It is noteworthy that the terms on the left-hand side of (\ref{derivative_eq}) are a weighted linear combination of $\bs{w}_{t}$'s and the corresponding weight on $\bs{w}_{t}$ is equal to $-\sum_{j=1}^K C_{j,t}P_{j} \ell'(-\left<\bs{w}_{t},\bs{f}^{*}\right>)$. Seeing that} $\sum_{t=1}^K \bs{w}_{t}=\bs{0}$ and $\sum_{j=1}^K P_{j}=1$, we might conclude that for any $t$,
		\begin{equation}\label{wight_eq}
		-\sum_{j=1}^K C_{j,t}P_{j} \ell'(-\left<\bs{w}_{t},\bs{f}^{*}\right>)=M,
		\end{equation}
		where $M$ is a positive {normalizing} constant that guarantees $\sum_{j=1}^K P_{j}=1$. From (\ref{wight_eq}), it is easy to verify that $\sum_{j=1}^K C_{j,t}P_{j}=\dfrac{-M}{\ell'(-\left<\bs{w}_{t},\bs{f}^{*}\right>)}.$ 
		
		To {further} calculate $M$, we first re-express (\ref{wight_eq}) in the matrix form which leads to $\bs{C}^{T}\bs{p}=-M\bs{\ell}^{*}$ with $\bs{\ell}^{*}\in\mathbb{R}^{K}$ being a vector whose the $k$th element is $\ell'(-\left<\bs{w}_{k},\bs{f}^{*}\right>)^{-1}$. {When $\bs{C}$ is invertible}, then we have 
		\begin{equation}\label{p_value}
		\bs{p}=-M(\bs{C}^{T})^{-1}\bs{\ell}^{*}.
		\end{equation}
		Because of the fact that $\bs{1}^{T}\bs{p}=1$ where $\bs{1}\in\mathbb{R}^{K}$ is a vector with each element $1$, we can conclude that $-M\bs{1}^{T}(\bs{C}^{T})^{-1}\bs{\ell}^{*}=1$ and thus
		\begin{equation}\label{M_value}
		M=\dfrac{-1}{\bs{1}^{T}(\bs{C}^{T})^{-1}\bs{\ell}^{*}}.
		\end{equation} 
		Combining this with (\ref{p_value}) leads to
		\begin{equation}\label{p_value_M}
		\bs{p}=\dfrac{(\bs{C}^{T})^{-1}\bs{\ell}^{*}}{\bs{1}^{T}(\bs{C}^{T})^{-1}\bs{\ell}^{*}}.
		\end{equation}
		
		Specifically, in cost-insensitive case where $C_{j,t}=\mathbb{I}(j\neq t)$, (\ref{M_value}) becomes 
		$$M=\dfrac{1-K}{\sum_{k=1}^K \ell'(-\left<\bs{w}_{k},\bs{f}^{*}\right>)^{-1}},$$ and according to (\ref{p_value_M}) we have
		\begin{equation}\label{cis_pt_eq}
		P_{t}=1+\dfrac{(1-K)\ell'(-\left<\bs{w}_{t},\bs{f}^{*}\right>)^{-1}}{\sum_{k=1}^K \ell'(-\left<\bs{w}_{k},\bs{f}^{*}\right>)^{-1}}
		\end{equation}
		for $t=1, 2,..., K$.
		%	\begin{equation}\nonumber
		%	\Big(\sum_{\substack{j=1\\j\neq t}}^K P_{j}\Big)\ell'(-\left<\bs{w}_{t},\bs{f}^{*}\right>)=-M.
		%	\end{equation}
		%	
		%	Combining this with $\sum_{j=1}^K P_{j}=1$ gives $M=\dfrac{1-K}{\sum_{k=1}^K \ell'(-\left<\bs{w}_{k},\bs{f}^{*}\right>)^{-1}}$ and 
		%	\begin{equation}\label{cis_pt_eq}
		%	P_{t}=1+\dfrac{(1-K)\ell'(-\left<\bs{w}_{t},\bs{f}^{*}\right>)^{-1}}{\sum_{k=1}^K \ell'(-\left<\bs{w}_{k},\bs{f}^{*}\right>)^{-1}}.
		%	\end{equation}
		
		Now, we will verify that $P_{t}\in (0,1)$ in (\ref{cis_pt_eq}).
		Since $\ell'$ {is} always less than zero and $K>1$, we can easily conclude that $P_{t}$ {is} less than $1$. The lower bound of $P_{t}$ will be verified by contradiction.
		
		If $P_{t}\leq 0$, it follows that
		\begin{equation}\label{Pt_0}
		\dfrac{(1-K)\ell'(-\left<\bs{w}_{t},\bs{f}^{*}\right>)^{-1}}{\sum_{k=1}^K \ell'(-\left<\bs{w}_{k},\bs{f}^{*}\right>)^{-1}}\leq -1
		\end{equation}
		Sum over $t$ on both sides of (\ref{Pt_0}), we have 
		$$\dfrac{(1-K)\sum_{t=1}^K\ell'(-\left<\bs{w}_{t},\bs{f}^{*}\right>)^{-1}}{\sum_{k=1}^K \ell'(-\left<\bs{w}_{k},\bs{f}^{*}\right>)^{-1}}\leq -K,$$ which finally leads to $1\leq 0$. Therefore, $P_{t}\in (0,1)$ in (\ref{cis_pt_eq}).
	\end{proof}
	
	From Theorem \ref{Th_2}, we could observe that the greater the $\left<\bs{w}_{t},\bs{f}^{*}\right>$, the lower the expected cost with the predicted class $t$. In practice, given the fitted $\hat{\bs{f}}$, one can replace $\bs{f}^{*}$ by $\hat{\bs{f}}$ to easily derive the estimated costs and class probabilities with the help of Theorem \ref{Th_2}.

	Theorems \ref{Th_1} and \ref{Th_2} indicate that plenty of large-margin loss functions could be generalized by angle-based method for cost-sensitive multicategory classification problem. In the reminder of this subsection, we mainly focus on the following three loss functions for detailed discussion.
	
	\subsubsection{Exponential Loss}
	First, we consider the exponential loss of the form $\ell(z)=e^{-z}$, and $\ell^{\prime}(z)=-e^{-z}$. Since $\ell(z)$ is convex in $z$, we can easily conclude that it could be extended to the angle-based cost-sensitive version as $\sum_{t=1}^K C_{y,t} e^{\left<\bs f(\bs{x}),\bs{w}_{t}\right>}$ for multicategory problem by Theorem \ref{Th_1}. 
	
	In addition, the corresponding expected cost with the class $t$ becomes $$\sum_{j=1}^K C_{j,t}P_{j}=Me^{-\left<\bs f^{*}(\bs{x}),\bs{w}_{t}\right>},$$ where $M$ is the normalizing constant defined in Theorem \ref{Th_2}. In the cost-insensitive case, the class probabilities can be expressed as
	\begin{equation}\nonumber
	P_{t}=1+\dfrac{(1-K)e^{-\left<\bs f^{*}(\bs{x}),\bs{w}_{t}\right>}}{\sum_{k=1}^K e^{-\left<\bs f^{*}(\bs{x}),\bs{w}_{k}\right>}},
	\end{equation}
	for $t=1, 2,..., K$. To express $\bs{f}^{*}$ in terms of the class probabilities, we get
	$$\left<\bs{f}^{*}(\bs{x}), \bs{w}_{t}\right>=\log(K-1)-\log(1-P_{t})-\log(\sum_{k=1}^K e^{-\left<\bs f^{*}(\bs{x}),\bs{w}_{k}\right>}).$$
	
	Sum the left-side of this equation over $t$, we conclude that $$0=K\log(K-1)-\sum_{t=1}^K\log(1-P_{t})-K\log(\sum_{k=1}^K e^{-\left<\bs f^{*}(\bs{x}),\bs{w}_{k}\right>}).$$
	
	Equivalently, we have
	\begin{equation}\label{exp_loss_fw}
		\left<\bs{f}^{*}(\bs{x}), \bs{w}_{t}\right>=\dfrac{1}{K}\sum_{k=1}^K\log(1-P_{k})-\log(1-P_{t}).
	\end{equation}
	
	{Particularly, for a classification problem with only $2$ classes, we have $\left< f^{*}(\bs{x}),{w}_{1}\right>=\dfrac{1}{2}\log\left(\frac{P_{1}}{P_{2}}\right)$ and $\left< f^{*}(\bs{x}),{w}_{2}\right>=\dfrac{1}{2}\log\left(\frac{P_{2}}{P_{1}}\right)$according to (\ref{exp_loss_fw}). It is noteworthy that this result is similar to the traditional binary classification methods, whcih verifies the rationality of the proposed approach.%这里要加一段类似于logit最后一段的说明吗 就是说在2为情况下和传统的分类形式相一致
	}

	\subsubsection{Logit Loss}
	The logit loss function is of the form $\ell(z)=\log (1+e^{-z})$ with the derivative $\ell^{\prime}(z)=-\dfrac{1}{1+e^{z}}$. From Theorem \ref{Th_1}, its angle-based cost-sensitive version for multicategory classification is $\sum_{t=1}^K C_{y,t} \log\left(1+e^{\left<\bs f(\bs{x}),\bs{w}_{t}\right>}\right)$.
	
	Accordingly, the expected cost for the class $t$ can be expressed as $$\sum_{j=1}^K C_{j,t}P_{j}=M\left(1+e^{-\left<\bs f^{*}(\bs{x}),\bs{w}_{t}\right>}\right)$$ with $M$ being a normalizing constant as defined in Theorem \ref{Th_2}. 
	
	In the cost-insensitive case, the class probabilities are given by
	\begin{equation}\nonumber
	P_{t}=1+\dfrac{(1-K)\left(1+e^{-\left<\bs f^{*}(\bs{x}),\bs{w}_{t}\right>}\right)}{\sum_{k=1}^K \left(1+e^{-\left<\bs f^{*}(\bs{x}),\bs{w}_{k}\right>}\right)}
	\end{equation}
	for $t=1, 2,..., K$. 
	
	To express the inner product by the conditional class probabilities, we have $$\left<\bs{f}^{*}(\bs{x}), \bs{w}_{t}\right>=-\log\left[\dfrac{\eta(1-P_{t})}{K-1}-1\right],$$ where $\eta=\sum_{k=1}^K \left(1+e^{-\left<\bs f^{*}(\bs{x}),\bs{w}_{k}\right>}\right)$ and it satisfies
	$$\sum_{t=1}^{K}\log\left[\dfrac{\eta(1-P_{t})}{K-1}-1\right]=0,$$
	since $\sum_{t=1}^{K}\left\langle \bs{f}^{*}(\bs{x}), \bs{w}_{t}\right\rangle=0$.
	
	When $K=2$, we can find that $\eta={(P_{1}P_{2})}^{-1}$ based on the above equation. Then we have $\left< f^{*}(\bs{x}),{w}_{1}\right>=\log\left(\frac{P_{1}}{P_{2}}\right)$ and $\left< f^{*}(\bs{x}),{w}_{2}\right>=\log\left(\frac{P_{2}}{P_{1}}\right)$. Note that this derives the familiar results of binary classification. However, the relationship between $\left<\bs f^{*}(\bs{x}),\bs{w}_{t}\right>$ and the class probabilities become more complex when $K>2$.
	%$\left<\bs f^{*}(\bs{x}),\bs{w}_{t}\right>$ depends on the class probabilities in a much more complex way.

	\subsubsection{Large-Margin Unified Machine Family}
	The large-margin unified machine uses the large-margin unified loss function \citep{Liu2011} which is given by 
	\begin{equation}\label{LMUM_Loss}
	\ell(z)=\left\{\begin{array}{ll}{1-z} & {\text{if  } z<\dfrac{c}{1+c},} \\ {\dfrac{1}{1+c}\left[\dfrac{a}{(1+c)z-c+a}\right]^{a}} & {\text{if  } z\geq\dfrac{c}{1+c},}\end{array}\right.
	\end{equation}
	where $c\geq 0$ and $a > 0$ are parameters of the large-margin unified machine family. {Also, its derivative $\ell^{\prime}(z)$ is given by
		\begin{equation}\label{LMUM_Loss_df}
		\ell^{\prime}(z)=\left\{\begin{array}{ll}{-1} & {\text{if  } z<\dfrac{c}{1+c},} \\ {-\left[\dfrac{a}{(1+c)z-c+a}\right]^{a+1}} & {\text{if  } z\geq\dfrac{c}{1+c}.}\end{array}\right.
		\end{equation}}
	The large-margin unified machine provides a bridge between soft and hard classifiers and connects them as a family \citep{Zhang2013}. Particularly, with $c = 0$, it leads to a typical soft classifier. When $c\rightarrow+\infty$, the large-margin unified machine loss tends to become the hinge loss which corresponds to a typical hard classifier. 
	
	It is obvious that the large-margin unified loss function with $c<+\infty$ satisfies the conditions in Theorem \ref{Th_1}, and thus its angle-based cost-sensitive {extension} $\sum_{t=1}^K C_{y,t} \ell(-\left<\bs f(\bs{x}), \bs{w}_{t}\right>)$ with $\ell(\cdot)$ defined in (\ref{LMUM_Loss}) is Fisher consistent.
	
	According to Theorem \ref{Th_2}, we may also conclude that for a predicted class $t$,  
	$$\sum_{j=1}^K C_{j,t}P_{j}=\dfrac{-M}{\ell'(-\left<\bs{w}_{t}, \bs{f}^{*}\right>)}$$
	with $\ell^{\prime}(\cdot)$ being the derivative defined in (\ref{LMUM_Loss_df}). Therefore, if $\left<\bs{w}_{i}, \bs{f}^{*}\right>>-\frac{c}{1+c}$ and $\left<\bs{w}_{j}, \bs{f}^{*}\right>>-\frac{c}{1+c}$ both hold for $\bs{f}^{*}$, then the  class conditional expected cost for classes $i$ and $j$ are equal.  When the value of $\left<\bs{w}_{i}, \bs{f}^{*}\right>$ is large and the value of $\left<\bs{w}_{j}, \bs{f}^{*}\right>$ is very small, we can verify that the expected cost for class $i$ is equal to $M$, and the expected cost for class $j$ is $\frac{M}{a^{a+1}}\left[-(1+c)\left<\bs{w}_{j}, \bs{f}^{*}\right>-c+a\right]^{a+1}$ whose value is obviously larger than $M$ since $\left<\bs{w}_{j}, \bs{f}^{*}\right><-\frac{c}{1+c}$. 
	
	In addition, the class probabilities in cost-insensitive case could be given by $$P_{t}=1+\dfrac{(1-K)\ell'(-\left<\bs{w}_{t}, \bs{f}^{*}\right>)^{-1}}{\sum_{k=1}^K \ell'(-\left<\bs{w}_{k}, \bs{f}^{*}\right>)^{-1}}$$ for $t=1, 2,..., K$. In this case, the class conditional probabilities for classes $i$ and $j$ are the same when $\bs{f}^{*}$ satisfies that both $\left<\bs{w}_{i}, \bs{f}^{*}\right>>-\frac{c}{1+c}$ and $\left<\bs{w}_{j}, \bs{f}^{*}\right>>-\frac{c}{1+c}$. %{\color{red}Moreover, when $\left<\bs{w}_{i}, \bs{f}^{*}\right>$  is large and  $\left<\bs{w}_{j}, \bs{f}^{*}\right>$  is small, we have $P_{i}=1$ and $P_{j}=0$. (question: This result is deduced because c-$>$ infinity or because $<>$ is large or small enough?? see 2014 ZHANG\&LIU)} 
	%comment: 最后红色的这一句我不太确定，因为原文我就不太确定到底是哪一个原因导致的概率等于1和0。
	
	\begin{remark}
		{Note that there are more loss functions that could be extended to their cost-sensitive multicategory versions with the help of Theorem \ref{Th_1}. We will not list them all here due to the lack of space. Because different loss functions lead to different classification methods, these methods could be directly generalized by applying the extended losses. To verify the usefulness of the proposed framework, we take exponential and logistic losses as examples and develop two novel cost-sensitive boosting algorithms for multicategory classification in the next section.}
		% 我们首先应该说还有更多的损失函数可以推广到这种情况，这里就不在赘述了。其次，各种损失函数被应用到非常的方法中，例如逻辑回归，boosting，svm等等。所以，这些方法都可以直接推广到cost-sensitive version。最后，我们为了展示方法的有效性选择boosting作为事例进行展示，下一节两种新的多分类boosting算法。
	\end{remark}
	
	\section{Cost-Sensitive Multicategory Boosting}\label{Sec4}
	{
		Boosting, as one of the most well-known learning methods, combines many “weak” classifiers to achieve better classification performance. Several attempts have been made to develop boosting algorithms in multiclass setting, such as AdaBoost.M2 \citep{Freund1997119}, AdaBoost.MH \citep{Schapire1999}, p-norm boosting \citep{Lozano2008}, and SAMME \citep{Zhu2009}. Afterthat, \cite{Wang2013}  also developed the multicategory boostings in cost-sensitive situation. This section aims to construct several new cost-sensitive boosting algorithms for multiclass classification {without the sum-to-zero constraint.}
		%Yi_1003:这句还要提18年发表在PR上的BAdaCost: Multi-class Boosting with Costs这篇文章吗？提了reviewer不会让咱们和它再比吧……
		%此外p-norm boosting本身也是cost-sensitive的，这里这样表述行吗？
	}

	%Chang: 这一节开头，先说boosting是和牛算法。然后被推广到多分类的情况，谁谁谁，zhuji什么的。然后，junhui推广到cost-sensitive的。之后才是Take advantage of ....
	% 这一节，符号中的M，与前面重要的constant M同样符号，最好选择某一个换一个哈。
	
	% 这一节中，adaboost需要说一下利用的additive model的想法推到了，要引用AOS那个paper。
	
	%对于，logit boost应该应用gradient boosting的paper说最后用的梯度法。

	%Give a direct application of exponential and logit losses in the angle-based cost-sensitive multicategory boosting to derive new algorithms. 
	
	\subsection{Cost-Sensitive AdaBoost}\label{CS_Ada}
	We first propose a new angle-based cost-sensitive AdaBoost algorithm for multicategory classification problem by using exponential loss. That is, we solve (\ref{ERM_our_loss}) with $\ell(z)=e^{-z}$ {and derive our Adaboost algorithm based on forward stagewise additive modeling scheme \citep{Friedman2000}.} Thus, the proposed angle-based cost-sensitive multiclass AdaBoost algorithm can be developed by solving
	\begin{equation}\nonumber
	\min_{\bs{f}\in \mathcal{F}} \dfrac{1}{n} \sum_{i=1}^{n} \sum_{k=1}^K C_{y_{i},k} e^{\left<\bs f(\bs{x}_{i}),\bs{w}_{k}\right>},
	\end{equation}
	where $\bs{f}(\bs{x})=\sum_{m=1}^{M}\beta^{(m)}\bs{g}^{(m)}(\bs{x})$ with $\bs{g}^{(m)} \in \mathcal{G}$, and $M$ is the prespecified number of boosting iterations.
	
	In order to find the optimal candidate to update the current model in each iteration, the gradient descent search scheme is applied. {To begin with,} we consider $\bs{g}(\bs{x})$ which takes values in one of the $K$ possible $(K-1)$-dimensional vectors in $\bs{W}$, i.e. $\bs{g}:\mathbb{R}^{d}\rightarrow\mathcal Y_{a}=\{\bs{w}_{1}, \bs{w}_{2}, ..., \bs{w}_{K}\}$. Note that for any $\bs{g}(\bs{x})$ defined in this manner, there exists a unique classification decision rule $\Phi(\bs{x}):\mathbb{R}^{d}\rightarrow\mathcal Y=\{1, 2, ..., K\}$ so that $\bs{g}$ {is in one-to-one correspondence with $\Phi$}. Given the current model $\bs{f}^{(m)}$, the gradient descent search scheme tries to find the optimal candidate function $\bs{g}^{(m+1)}$ and corresponding coefficient $\beta^{(m+1)}$ through
	\begin{align}\label{GDS_frame}
	\left(\bs{g}^{(m+1)}, \beta^{(m+1)}\right) &=\arg \min _{\bs{g}, \beta} \sum_{i=1}^{n} \sum_{k=1}^{K} C_{y_{i}, k} e^ {\langle\bs{f}^{(m)}(\bs{x}_{i}), \bs{w}_{k}\rangle+\beta \left\langle\bs{g}(\bs{x}_{i}), \bs{w}_{k}\right\rangle} \\\label{GDS_ERM} &=\arg \min _{\bs{g}, \beta} \sum_{i=1}^{n} \sum_{k=1}^{K} \alpha_{i, k}^{(m)} e^{\beta \left\langle\bs{g}(\bs{x}_{i}), \bs{w}_{k}\right\rangle}, 
	\end{align}
	where $\alpha_{i, k}^{(m)}=C_{y_{i}, k} e^ {\langle\bs{f}^{(m)}(\bs{x}_{i}), \bs{w}_{k}\rangle}$ scales the cost of misclassifying example $(\bs{x}_{i}, y_{i})$ into class $k$ by a weighting factor $e^ {\langle\bs{f}^{(m)}(\bs{x}_{i}), \bs{w}_{k}\rangle}$. Notice that solving for $\bs{g}^{(m+1)}(\bs{x})$ in (\ref{GDS_ERM}) is equivalent to finding the corresponding $\Phi^{(m+1)}(\bs{x})$ since $\bs{g}$ and $\Phi$ have a one-to-one correspondence. Then, the current model could be updated by $\bs{f}^{(m+1)}(\bs{x})=\bs{f}^{(m)}(\bs{x})+\beta^{(m+1)} \bs{g}^{(m+1)}(\bs{x})$ with the help of the following lemma.
	\begin{lemma}\label{L1}
		The solution to (\ref{GDS_ERM}) is
		\begin{equation}
		\begin{aligned} 
		&\Phi^{(m+1)}(\bs{x}) =\arg \min _{\Phi} \sum_{i=1}^{n} \alpha_{i, \Phi\left(\bs{x}_{i}\right)}^{(m)}, \\ 
		&\beta^{(m+1)} =\frac{K-1}{K}\left[\log \left(\frac{1-\varepsilon^{(m+1)}}{\varepsilon^{(m+1)}}\right)-\log (K-1)\right],
		\end{aligned}
		\end{equation}
		where ${\varepsilon^{(m+1)}}$ is defined as
		\begin{equation}\label{err_equ}
		\varepsilon^{(m+1)}=\frac{\sum_{i=1}^{n} \alpha_{i, \Phi^{(m+1)}\left(\bs{x}_{i}\right)}^{(m)}}{\sum_{i=1}^{n} \sum_{k=1}^{K} \alpha_{i, k}^{(m)}}.
		\end{equation}
		
	\end{lemma}
	
	\begin{proof}
		Since there is a one-to-one correspondence between $\bs{g}(\bs{x})$ and $\Phi(\bs{x})$, we could replace $\bs{g}(\bs{x})$ by its corresponding $\Phi(\bs{x})$ in (\ref{GDS_ERM}). Note that $\alpha_{i, y_{i}}^{(m)}=0$, then we could obtain the following equivalent optimization problem
		\begin{align}\nonumber
		&\arg \min _{\Phi, \beta} \sum_{\{y_{i}=\Phi(\bs{x}_{i})\}} \sum_{k=1}^{K} \alpha^{(m)}_{i, k}e^{\beta\langle\bs{w}_{y_{i}}, \bs{w}_{k}\rangle}+\sum_{\{y_{i} \neq \Phi(\bs{x}_{i})\}}\left(\alpha^{(m)}_{i, \Phi(\bs{x}_{i})}e^{\beta}+ \sum_{k \neq \Phi(\bs{x}_{i})} \alpha^{(m)}_{i, k}e^{\beta\langle\bs{w}_{\Phi(\bs{x}_{i})}, \bs{w}_{k}\rangle}\right) \\\label{GDS_ERM_2} =&\arg \min _{\Phi, \beta}  e^{\beta \cos{\theta_{K}}} \sum_{i=1}^{n} \sum_{k=1}^{K} \alpha^{(m)}_{i, k}+\left(e^{\beta}-e^{\beta\cos{\theta_{K}}}\right) \sum_{i=1}^{n} \alpha^{(m)}_{i, \Phi(\bs{x}_{i})},
		\end{align}
		where $\theta_{K}$ is the angle between any two different vectors in $\bs{W} =\{\bs{w}_{1}, ..., \bs{w}_{K}\}$. 
		
		Since only the second term depends on $\Phi(\bs{x})$ and $e^{\beta}-e^{\beta\cos{\theta_{K}}}>0$ for $K\geq2$, solving for $\bs{g}^{(m+1)}(\bs{x})$ in (\ref{GDS_ERM_2}) is equivalent to searching for$$\Phi^{(m+1)}(\bs{x}) =\arg \min _{\Phi} \sum_{i=1}^{n} \alpha_{i, \Phi\left(\bs{x}_{i}\right)}^{(m)}.$$
		
		{After obtaining the optimal $\Phi^{(m+1)}(\bs{x})$, we plug it into (\ref{GDS_ERM_2}), then the optimization problem reduces to}
		$$\arg \min _{\beta} R(\beta)=e^{\beta \cos{\theta_{K}}} +\left(e^{\beta}-e^{\beta\cos{\theta_{K}}}\right)  \varepsilon^{(m+1)},
		$$
		where $$\varepsilon^{(m+1)}=\frac{\sum_{i=1}^{n} \alpha_{i, \Phi^{(m+1)}\left(\bs{x}_{i}\right)}^{(m)}}{\sum_{i=1}^{n} \sum_{k=1}^{K} \alpha_{i, k}^{(m)}}.$$
		
		Since $R(\beta)$ is convex in $\beta$, taking the derivative of $R(\beta)$ yields
		$$\frac{\partial R(\beta)}{\partial \beta} =\cos{\theta_{K}}e^{\beta \cos{\theta_{K}}}+\left(e^{\beta}-\cos{\theta_{K}}e^{\beta \cos{\theta_{K}}}\right)\varepsilon^{(m+1)}.
		$$ 
		
		Setting this derivative equal to zero, we have
		\begin{align}\nonumber
		\beta^{(m+1)}&=\dfrac{1}{1-\cos{\theta_{K}}}\left[\log \left(\frac{1-\varepsilon^{(m+1)}}{\varepsilon^{(m+1)}}\right)+\log (-\cos{\theta_{K}})\right]
		\\\nonumber&=\frac{K-1}{K}\left[\log \left(\frac{1-\varepsilon^{(m+1)}}{\varepsilon^{(m+1)}}\right)-\log (K-1)\right].
		\end{align}
		The last equation is derived based on the fact that $\cos{\theta_{K}}=\frac{1}{1-K}$ accroding to (\ref{W_j_def}).	
	\end{proof}
	
	We could easily verify that $\beta^{(m+1)} > 0$ if $\varepsilon^{(m+1)}<\frac{1}{K}$  for any $\alpha_{i, k}^{(m)}$. {One can find that this condition is equivalent to $\sum_{i=1}^{n} \alpha_{i, \Phi^{(m+1)}\left(\bs{x}_{i}\right)}^{(m)}<\frac{1}{K}\sum_{i=1}^{n} \sum_{k=1}^{K} \alpha_{i, k}^{(m)}$ according to (\ref{err_equ}). Since $\frac{1}{K}\sum_{i=1}^{n} \sum_{k=1}^{K} \alpha_{i, k}^{(m)}$ on the right hand side measures the expected weighted misclassification cost of random guessing, this implies that $\Phi^{(m+1)}(\bs{x})$ only needs to perform better than random guessing class labels.}
	
	Based on Lemma \ref{L1}, we can derive the angle-based cost-sensitive AdaBoost algorithm for multi-class classification that is outlined in Algorithm \ref{Algo_Ada}.
	\begin{remark}
		{The proposed Algorithm \ref{Algo_Ada} is very similar to the MultiBoost developed by \cite{Wang2013}. However, since the simplex class coding is applied in our algorithm, we utilize different candidate function $\bs{g}(\bs{x})$ and the least-angle prediction rule  instead of the max rule used in \cite{Wang2013}. } %Yi_1003:老师您看这样写够了吗？还需要补充什么吗？
	\end{remark}
	
	\begin{algorithm} 
		\caption{Angle-Based Cost-Sensitive Multicategory AdaBoost}
		\begin{algorithmic}[1]\label{Algo_Ada}
			
			\REQUIRE Training set $\mathcal D=\{(\bs{x}_i,y_i)\}_{i=1}^n$ where $y_{i}\in\{1, 2, \ldots, K\}$ is the class label of example $\bs{x}_{i}$, cost matrix $\bs{C}$, and number $M$ of weak learners in the final decision function.\\
			\ENSURE Multicategory classifier $\bs{f}(\bs{x}).$ \\
			
			\STATE compute $\bs{W} =\{\bs{w}_{1}, ..., \bs{w}_{K}\}$ according to (\ref{W_j_def});
			\STATE initialize $$\alpha_{i, k}^{(0)}=\dfrac{C_{y_{i}, k}}{\sum_{i=1}^{n} \sum_{k=1}^{K} C_{y_{i}, k}}, \text { for } i=1, \ldots, n \text { and } k=1, \ldots, K;$$
			
			\FOR{$m = 0$ to $M-1$}
			\STATE solve (\ref{GDS_ERM}) as in Lemma \ref{L1} for $\left(\Phi^{(m+1)}, \beta^{(m+1)}\right)$;
			\STATE convert $\Phi^{(m+1)}(\bs{x})$ to $\bs{g}^{(m+1)}(\bs{x})$ and update $\bs{f} ^{(m+1)}(\bs{x}) = \bs{f} ^{(m)}(\bs{x})+
			\beta^{(m+1)}\bs{g}^{(m+1)}(\bs{x})$;
			\STATE update $\alpha_{i, k}^{(m+1)}=\alpha_{i, k}^{(m)}e^{\beta^{(m+1)}\langle\bs{g}^{(m+1)}(\bs{x}_{i}), \bs{w}_{k}\rangle}$;
			\STATE renormalize $\alpha_{i, k}^{(m+1)}$ by dividing it by $\sum_{i=1}^{n} \sum_{k=1}^{K} \alpha_{i, k}^{(m+1)}$;
			\ENDFOR
			\STATE compute $\bs{f}(\bs{x})=\sum_{m=1}^{M}\beta^{(m)}\bs{g}^{(m)}(\bs{x})$;
			
			\RETURN $\bs{f}(\bs{x})$.	
			
		\end{algorithmic}	
	\end{algorithm}
	
	\subsection{Cost-Sensitive LogitBoost.ML}
	By solving (\ref{ERM_our_loss}) with logit loss, we propose a novel cost-sensitive logit boosting algorithm by using the angle-based framework. Given the prespecified number of iterations $M$, the optimal classifier $\bs{f}(\bs{x})$ could be derived by minimizing
	
	\begin{equation}\label{Logit_ERM}
	\dfrac{1}{n} \sum_{i=1}^{n} \sum_{k=1}^K C_{y_{i},k} \log \left(1+e^{\left<\bs f(\bs{x}_{i}),\bs{w}_{k}\right>}\right),
	\end{equation}
	where $\bs{f}(\bs{x})=\sum_{m=1}^{M}\beta^{(m)}\bs{g}^{(m)}(\bs{x})$ with $\bs{g}^{(m)} \in \mathcal{G}$. {The gradient decent method is used here to search for the optimal $\bs{f}(\bs{x})$ \citep{Friedman2001, Zou2008}.}
	
	Note that given the current fit $\bs{f}^{(m)}(\bs{x})$,  the negative gradient of (\ref{Logit_ERM}) is equal to
	
	\begin{equation}\nonumber
	-{\dfrac{1}{n}}\sum_{k=1}^K  \dfrac{C_{y_{i},k}e^{\left<\bs f(\bs{x}_{i}),\bs{w}_{k}\right>}}{1+e^{\left<\bs f(\bs{x}_{i}),\bs{w}_{k}\right>}}\bs{w}_{k}
	\end{equation}
	for $i=1, 2, \ldots, n$. In order to find the optimal incremental direction $\bs{g}^{(m+1)}(\bs{x})$ that best approximates the negative gradient direction, we need to solve the following optimization problem:
	
	\begin{align}\nonumber
	\arg \max_{\bs{g}}\quad &\sum_{i}^{n} \sum_{k=1}^K  -\alpha^{(m)}_{i, k}\left\langle\bs{g}(\bs{x}_{i}), \bs{w}_{k}\right\rangle\\\label{constraint1}\text {subject to} \quad&\sum_{k=1}^{K} \left\langle \bs{g}, \bs{w}_{k} \right\rangle^{2}=\dfrac{K}{K-1}
	\end{align}
	with $\alpha^{(m)}_{i, k}=\dfrac{C_{y_{i},k}e^{\left<\bs f^{(m)}(\bs{x}_{i}),\bs{w}_{k}\right>}}{1+e^{\left<\bs f^{(m)}(\bs{x}_{i}),\bs{w}_{k}\right>}}$. The right-hand side of (\ref{constraint1}) is set arbitrarily since this constraint is imposed here only to insure that the value of $\left\langle \bs{g}, \bs{w}_{k} \right\rangle$ is bounded. 
	
	Similar to Subsection \ref{CS_Ada}, we still consider  $\bs{g}(\bs{x}):\mathbb{R}^{d}\rightarrow\mathcal Y_{a}=\{\bs{w}_{1}, \bs{w}_{2}, ..., \bs{w}_{K}\}$ and its corresponding decision rule $\Phi(\bs{x}):\mathbb{R}^{d}\rightarrow\mathcal Y=\{1, 2, ..., K\}$. With this setting, the equality constraint in (\ref{constraint1}) is satisfied, and the optimal candidate function $\bs{g}^{(m+1)}(\bs{x})$ could be found through 
	\begin{equation}\label{GDS_Frame_Logit}
	\bs{g}^{(m+1)}(\bs{x})=\arg \min_{\bs{g}} \sum_{i}^{n} \sum_{k=1}^K \alpha^{(m)}_{i, k}\left\langle\bs{g}(\bs{x}_{i}), \bs{w}_{k}\right\rangle.
	\end{equation}

	Since there exists a one-to-one correspondence between $\bs{g}(\bs{x})$ and $\Phi(\bs{x})$, we have the following lemma.
	
	\begin{lemma}\label{L2}
		The solution of (\ref{GDS_Frame_Logit}) is given by $
		\Phi^{(m+1)}(\bs{x}) =\arg \min _{\Phi} \sum_{i=1}^{n} \alpha_{i, \Phi\left(\bs{x}_{i}\right)}^{(m)}$.
	\end{lemma}
	
	\begin{proof}
		
		Replacing $\bs{g}(\bs{x})$ by its corresponding $\Phi(\bs{x})$ in (\ref{GDS_Frame_Logit}) leads to the following equivalent optimization problem  
		\begin{align}\nonumber
		&\arg \min _{\Phi} \sum_{\{y_{i}=\Phi(\bs{x}_{i})\}} \sum_{k=1}^{K} \alpha^{(m)}_{i, k}\langle\bs{w}_{y_{i}}, \bs{w}_{k}\rangle+\sum_{\{y_{i} \neq \Phi(\bs{x}_{i})\}}\left(\alpha^{(m)}_{i, \Phi(\bs{x}_{i})}+ \sum_{k \neq \Phi(\bs{x}_{i})} \alpha^{(m)}_{i, k}\langle\bs{w}_{\Phi(\bs{x}_{i})}, \bs{w}_{k}\rangle\right) \\\label{GDS_ERM_3} =&\arg \min _{\Phi} \cos{\theta_{K}} \sum_{i=1}^{n} \sum_{k=1}^{K} \alpha^{(m)}_{i, k}+\left(1-\cos{\theta_{K}}\right) \sum_{i=1}^{n} \alpha^{(m)}_{i, \Phi(\bs{x}_{i})}.
		\end{align}
		
		Since only the second term is dependent of $\Phi(\bs{x})$ and $1-\cos{\theta_{K}}>0$ for $K>2$, solving (\ref{GDS_Frame_Logit}) for $\bs{g}^{(m+1)}(\bs{x})$ is equivalent to solving
		$$\Phi^{(m+1)}(\bs{x}) =\arg \min _{\Phi} \sum_{i=1}^{n} \alpha_{i, \Phi\left(\bs{x}_{i}\right)}^{(m)}.$$	
		This completes the proof.
	\end{proof}
	
	Based on Lemma \ref{L2}, we can induce $\bs{g}^{(m+1)}(\bs{x})$ from the fitted classifier $\Phi^{(m+1)}(\bs{x})$. Then the step length $\beta^{(m+1)}$ could be calculated by solving
	\begin{equation}\label{beta_logit}
	\arg\min_{\beta} R(\beta)=\sum_{i=1}^{n} \sum_{k=1}^{K}C_{y_{i},k}\log\left(1+\gamma^{(m)}_{i,k}e^{\beta\left\langle\bs{g}^{(m+1)}(\bs{x}_{i}), \bs{w}_{k}\right\rangle}\right),
	\end{equation}
	where $\gamma^{(m)}_{i,k}=e^{\langle\bs{f}^{(m)}(\bs{x}_{i}), \bs{w}_{k}\rangle}$. However, it is difficult to obtain its analytic solution. So several commonly used optimization algorithms could be applied here to find the optimal $\beta^{(m+1)}$, such as quasi-Newton method. Afterwards, the current model could be updated by $\bs{f}^{(m+1)}(\bs{x})=\bs{f}^{(m)}(\bs{x})+\beta^{(m+1)} \bs{g}^{(m+1)}(\bs{x})$.

	The angle-based cost-sensitive LogitBoost algorithm for multi-class classification problem is outlined in Algorithm \ref{Algo_Logit}.
	
	\begin{algorithm} 
		\caption{Angle-Based Cost-Sensitive Multicategory LogitBoost}
		\begin{algorithmic}[1]\label{Algo_Logit}
			
			\REQUIRE Training set $\mathcal D=\{(\bs{x}_i,y_i)\}_{i=1}^n$ where $y_{i}\in\{1, 2, \ldots, K\}$ is the class label of example $\bs{x}_{i}$, cost matrix $\bs{C}$, the initial value $\beta_{0}$, and number $M$ of weak learners in the final decision function.\\
			\ENSURE Multicategory classifier $\bs{f}(\bs{x}).$ \\
			
			\STATE compute $\bs{W} =\{\bs{w}_{1}, ..., \bs{w}_{K}\}$ according to (\ref{W_j_def});
			\STATE initialize \begin{align}\nonumber
			\alpha_{i, k}^{(0)}=\dfrac{C_{y_{i}, k}}{\sum_{i=1}^{n} \sum_{k=1}^{K} C_{y_{i}, k}} \quad\text{and}\quad
			\gamma^{(0)}_{i,k}=1%\dfrac{\alpha_{i, k}^{(0)}}{C_{y_{i}, k}-\alpha_{i, k}^{(0)}}
			\end{align}
			for $i=1, \ldots, n$ and $k=1, \ldots, K$;
			\FOR{$m = 0$ to $M-1$}
			\STATE solve (\ref{GDS_Frame_Logit}) as in Lemma \ref{L2} for $ \Phi^{(m+1)}$;
			\STATE convert $\Phi^{(m+1)}(\bs{x})$ to $\bs{g}^{(m+1)}(\bs{x})$;% and update $\bs{f} ^{(m+1)}(\bs{x}) = \bs{f} ^{(m)}(\bs{x})+
			%\beta^{(m+1)}\bs{g}^{(m+1)}(\bs{x})$;
			%\REPEAT
			\STATE compute $\beta^{(m+1)}$ as shown in (\ref{beta_logit}); %until it converges;
			%$$    	\beta^{(m+1)}\leftarrow\beta^{(m+1)}-\dfrac{\sum_{i=1}^{n} \sum_{k=1}^{K}\dfrac{C_{y_{i},k}\gamma^{(m)}_{i,k}\left\langle\bs{g}^{(m+1)}(\bs{x}_{i}), \bs{w}_{k}\right\rangle}{\gamma^{(m)}_{i,k}+e^{-\beta\left\langle\bs{g}^{(m+1)}(\bs{x}_{i}), \bs{w}_{k}\right\rangle}}}{\sum_{i=1}^{n} \sum_{k=1}^{K}\dfrac{C_{y_{i},k}\gamma^{(m)}_{i,k}}{e^{\beta\left\langle\bs{g}^{(m+1)}(\bs{x}_{i}), \bs{w}_{k}\right\rangle}}\left[\dfrac{\left\langle\bs{g}^{(m+1)}(\bs{x}_{i}), \bs{w}_{k}\right\rangle}{\gamma^{(m)}_{i,k}+e^{-\beta\left\langle\bs{g}^{(m+1)}(\bs{x}_{i}), \bs{w}_{k}\right\rangle}}\right]^{2}}$$
			%\UNTIL {it converges}
			\STATE update $\gamma_{i, k}^{(m+1)}=\gamma_{i, k}^{(m)}e^{\beta^{(m+1)}\langle\bs{g}^{(m+1)}(\bs{x}_{i}), \bs{w}_{k}\rangle}$ and $\alpha_{i, k}^{(m+1)}=\dfrac{C_{y_{i},k}\gamma_{i, k}^{(m+1)}}{1+\gamma_{i, k}^{(m+1)}}$;
			\STATE renormalize $\alpha_{i, k}^{(m+1)}$ by dividing it by $\sum_{i=1}^{n} \sum_{k=1}^{K} \alpha_{i, k}^{(m+1)}$;
			\ENDFOR
			\STATE compute $\bs{f}(\bs{x})=\sum_{m=1}^{M}\beta^{(m)}\bs{g}^{(m)}(\bs{x})$;
			
			\RETURN $\bs{f}(\bs{x})$.	
			
		\end{algorithmic}	
	\end{algorithm}
	
	%\subsection{Cost-Sensitive GentleBoost}
	%To be continued...

	\section{Experiment Study}\label{Sec5}
	
	In this section, we evaluate the performance of the proposed angle-based cost-sensitive boosting algorithms both on synthetic and real datasets.
	
	\subsection{Numerical Experiments}
	To verify the effectiveness of the proposed algorithms, two simulated examples are designed in this subsection. We compare the proposed algorithms, namely Angle-Based Adaboost and Angle-Based Logitboost, with {AdaBoost.M2 \citep{Freund1997119}, AdaBoost.MH \citep{Schapire1999}, p-norm boosting \citep{Lozano2008}, and SAMME \citep{Zhu2009}} algorithms both in cost-insensitive and cost-sensitive scenarios. The number of boosting steps is set as $200$ and classification trees \citep{Breiman1984} with at most $4$ terminal nodes are used as base learners in all algorithms and examples.%Yi_1005: 我不太明白wang junhui的这个the dictionary of candidate models是指什么，是指base learners吗？
	
	We use the test cost averaged over 100 independent simulation replications to evaluate the classification performance of each algorithm, which is defined as  
	\begin{equation}\nonumber%\label{measure_MC}
	{\rm TC}(\Phi)=\dfrac{1}{n_{t}}\sum_{i=1}^{n_{t}}\sum_{y_{i}\neq k} C_{y_{i},k}\mathbb{I}\left(\Phi(\bs{x}_{i})=k\right)
	\end{equation}
	with $n_{t}$ being the size of a test set. Besides, in all cost-insensitive scenarios, $C_{j,k}=\mathbb{I}(j\neq k)$ is applied. Specifically, the {BFGS method~\citep{Nocedal2006}} is utilized in Angle-Based Logitboost to search for the optimal step length. 
	
	\subsubsection{Simulation 1}
	We first apply a popular simulation example used in \cite{Breiman1984}, \cite{Zhu2009}, and \cite{Wang2013}, which is a three-class
	problem with $21$ features. In this simulation, a random sample $(\bs{x}_{i}, y_{i})$ with $i= 1, ..., 5000$ is generated independently from $y_{i}\sim {\rm uniform}\{1, 2, 3\}$ and $\bs{x}_{i}$ with $x_{ij}\sim N(\mu(y_{i},j), 1)$ where
	\begin{equation}\nonumber
	\mu(y_{i},j)=\left\{\begin{array}{ll}{u \cdot v_{1}(j)+(1-u) \cdot v_{2}(j) } & {\quad\text{if  } y_{i}=1,} \\ {u \cdot v_{1}(j)+(1-u) \cdot v_{3}(j) } & {\quad\text{if  } y_{i}=2,}\\
	{u \cdot v_{2}(j)+(1-u) \cdot v_{3}(j)} & {\quad\text{if  } y_{i}=3,}
	\end{array}\right.
	\end{equation}
	with $j= 1, ..., 21$, $u\sim {\rm uniform (0, 1)}$, and $v_{l}$ being the shifted triangular waveforms: $v_{1}(j)=\max (6-|j-11|,0)$, $v_{2}(j)=v_{1}(j-4)$ and $v_{3}(j)=v_{1}(j+4)$. The training set is chosen to be of size $300$ and the test set is of size $4700$. For cost-sensitive scenario, the misclassification cost matrix is set as in \cite{Wang2013}, where 
	\begin{equation}\nonumber
	\bs{C}=\left[\begin{array}{ll}{0 \quad 2 \quad 2} \\ {1 \quad 0 \quad 1}\\{1 \quad 1 \quad 0}\end{array}\right].
	\end{equation}
	
	\subsubsection{Simulation 2}
	
	In the second experiment, the simulation example proposed by \cite{Wang2013} is applied. This is a four-class problem with $10$ features. A random sample $(\bs{x}_{i}, y_{i})$ with $i= 1, ..., 5000$ is generated independently from $y_{i}\sim {\rm uniform}\{1, 2, 3, 4\}$ and $\bs{x}_{i}$ with $x_{i1}\sim N(\mu_{1}(y_{i}), 1)$, $x_{i2}\sim N(\mu_{2}(y_{i}), 1)$ and $x_{ij}\sim N(0, 1)$ for $j=3, ..., 10$, where
	\begin{align}\nonumber
	\mu_{1}(y_{i})&=3\left(\mathbb{I}(y_{i}=1)-\mathbb{I}(y_{i}=3)\right),\\\nonumber\mu_{2}(y_{i})&=3\left(\mathbb{I}(y_{i}=2)-\mathbb{I}(y_{i}=3)\right).
	\end{align}
	The size of training set and test set is still chosen to be $300$ and $4700$,  respectively. The misclassification cost matrix is also set as in \cite{Wang2013}, where
	$$
	\bs{C}=\left[
	\begin{matrix}
	0 & 1 & 2 &2 \\
	1 & 0 & 2 &2 \\
	0.5 & 0.5 & 0 & 1\\
	0.5 & 0.5 & 1 & 0
	\end{matrix}
	\right].$$
	
	Table \ref{SimuRes_table} shows the averaged test costs and their estimated standard errors over $100$ simulation replications of two simulated examples. These results clearly show that the two proposed angle-based boosting algorithms work well and are very competitive compared with other algorithms in all simulations. Especially in the simulated example 2, the proposed angle-based Logitboost achieves the lowest test costs in both cost-insensitive and cost-sensitive scenarios. In addition, Figures \ref{fig:1} and \ref{fig:2} display the test cost curves of all multi-class boosting algorithms as functions of boosting steps in two simulations. We can easily find that the test costs of all multi-class boosting algorithms decrease steadily as the number of iterations increases and then they stay almost flat, except that the test costs of AdaBoost.MH are generally growing after they achieve their minimums. In addition, because a fixed small step size is applied according to Lozano and Abe (2008), the decay speed of AdaBoost.M2 and p-norm boosting is relatively slower compared with other boosting algorithms.

	\begin{table} 
		\captionsetup{font={footnotesize}}
		\caption{Averaged test costs with their estimated standard errors inside parentheses for all boosting algorithms based on 100 replications of two simulations. The `c' in the first column indicates cost-sensitive scenario. Bold values represent the lowest averaged test cost achieved for each simulation.}
		\label{SimuRes_table}
		\footnotesize
		\centering
		\resizebox{!}{3.4cm}{
			\begin{tabular}{lcccccc}
				\toprule
				& {AdaBoost.MH}& {SAMME} & {AdaBoost.M2} & {p-norm Boost}&\thead[c]{Angle-Based \\Adaboost}&\thead[c]{Angle-Based\\ Logitboost}\\
				\midrule
				Simulation 1  & 0.252 & \textbf{0.182} & 0.264 & 0.211 & 0.201 & 0.204 \\
				& (0.0048) & \textbf{(0.0008)} & (0.0018)  & (0.0012)  & (0.0010)   &(0.0010)  \\
				Simulation 1c  & 0.300 & \textbf{0.237} & 0.389 & 0.293 &0.246 & 0.248 \\
				& (0.0072)& \textbf{(0.0014)} & (0.0021) & (0.0020) & (0.0014) & (0.0014)\\
				\hline
				Simulation 2  & 0.363 & 0.150 & 0.141 & 0.123 & 0.101 & \textbf{0.098} \\
				& (0.0058) &(0.0051) & (0.0051) & (0.0022) & (0.0007) & \textbf{(0.0006)}  \\
				Simulation 2c  &0.390&0.177&0.275&0.164&0.106&\textbf{0.100}\\
				&(0.0127)&(0.0054)&(0.0015)&(0.0020)&(0.0011)&(\textbf{0.0010)}    \\
				
				\bottomrule      
		\end{tabular}}
		\label{bs}
	\end{table}

	\begin{figure*}[!h]
		\centering
		\subfloat[Cost-Insensitive Scenario]{
			\includegraphics[width=200pt]{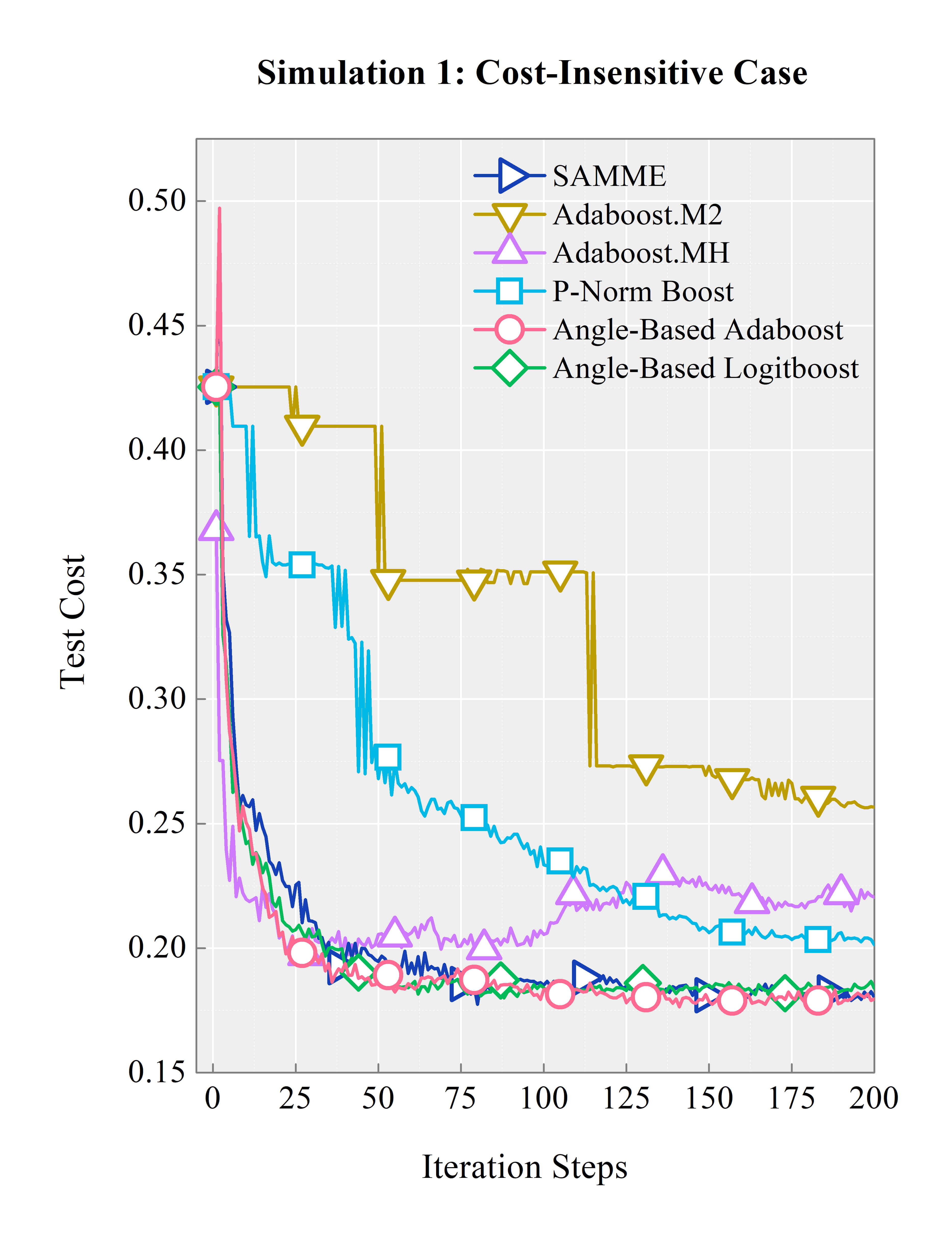}}
		\hspace{4pt}
		\subfloat[Cost-Sensitive Scenario]{
			\includegraphics[width=200pt]{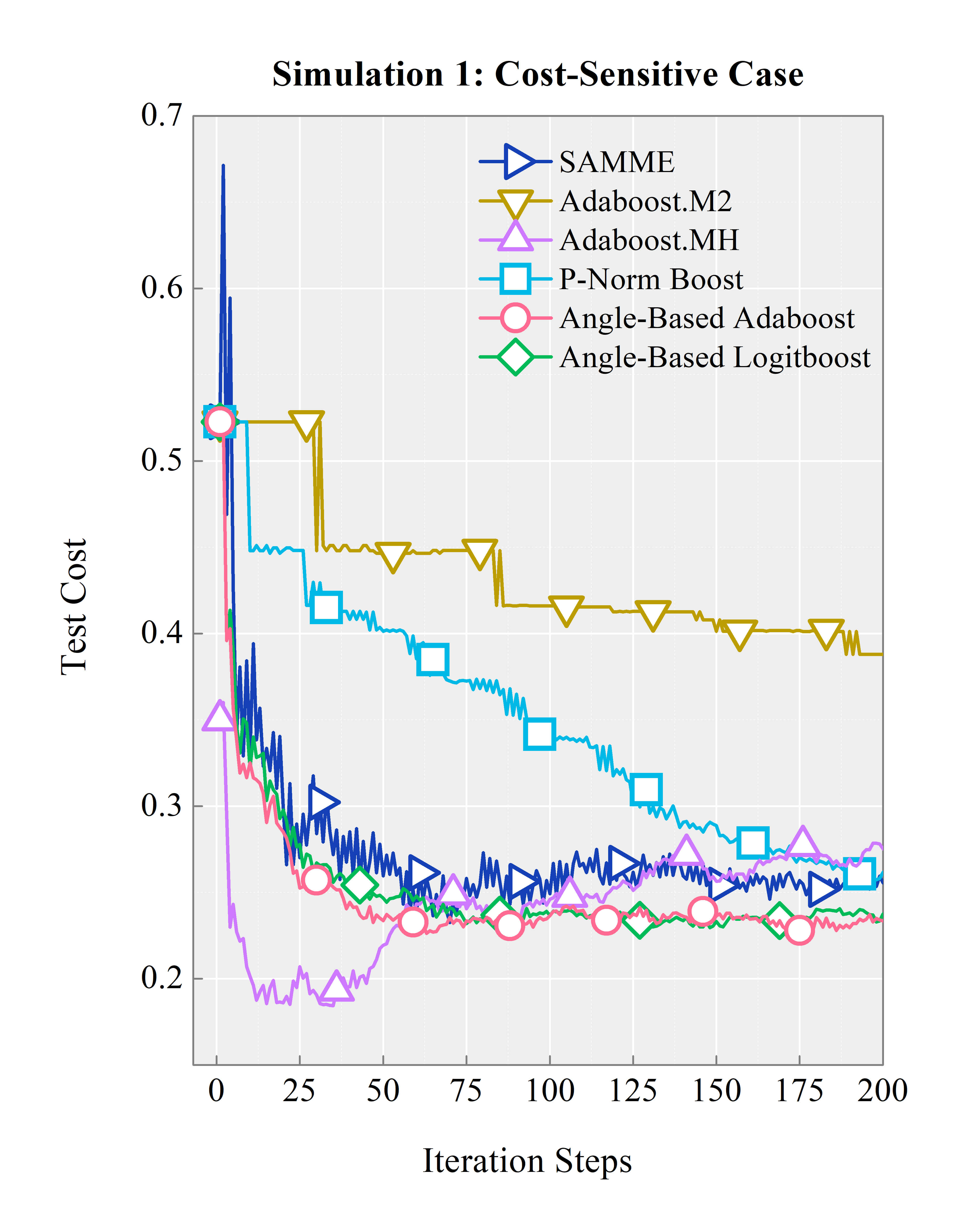}}
		\caption{Test costs of various boosting algorithms as a function of boosting steps in a randomly selected replication in Simulation 1.}
		\label{fig:1}
	\end{figure*}
	
	\begin{figure*}[h]
		\centering
		\subfloat[Cost-Insensitive Scenario]{
			\includegraphics[width=200pt]{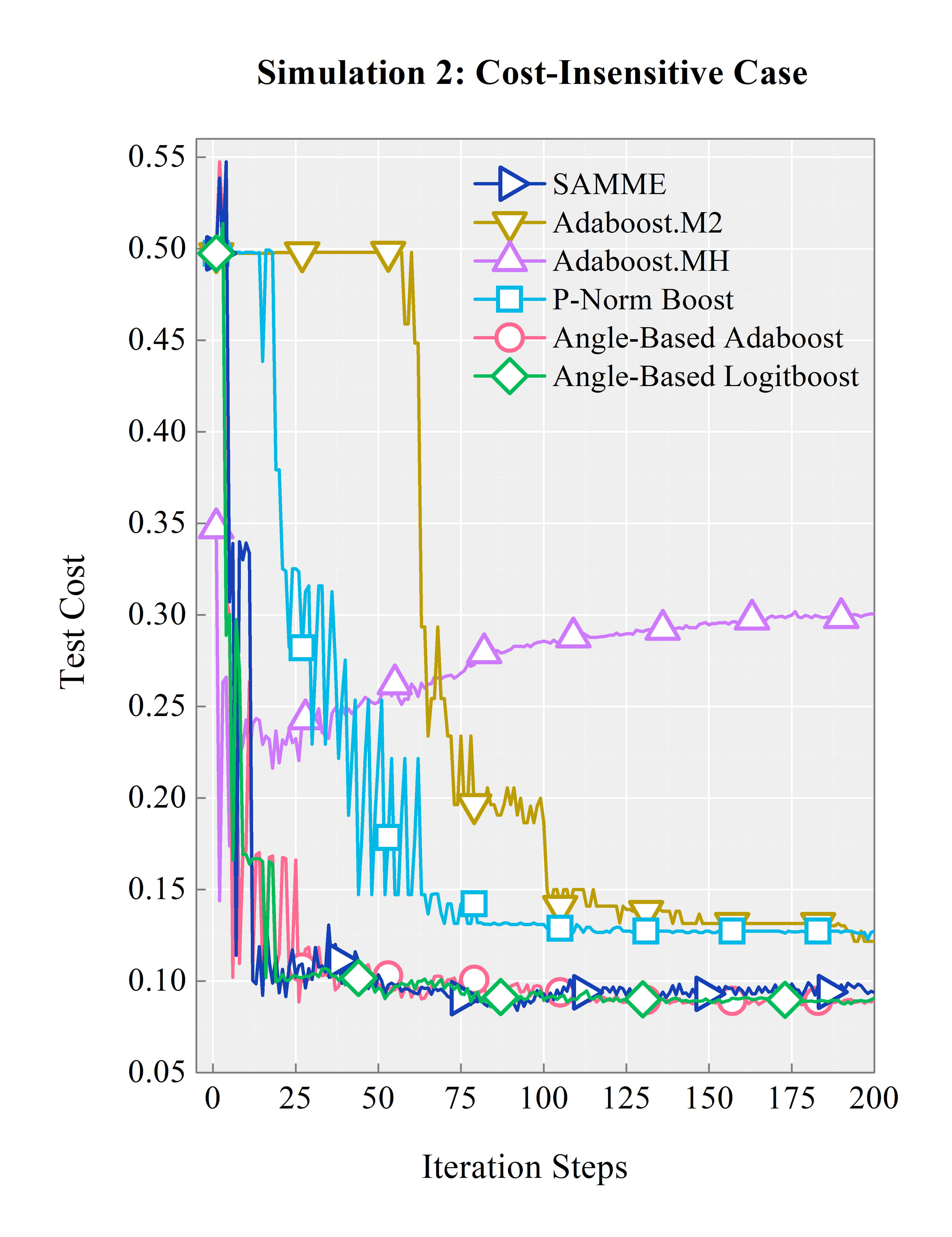}}
		\hspace{4pt}
		\subfloat[Cost-Sensitive Scenario]{
			\includegraphics[width=200pt]{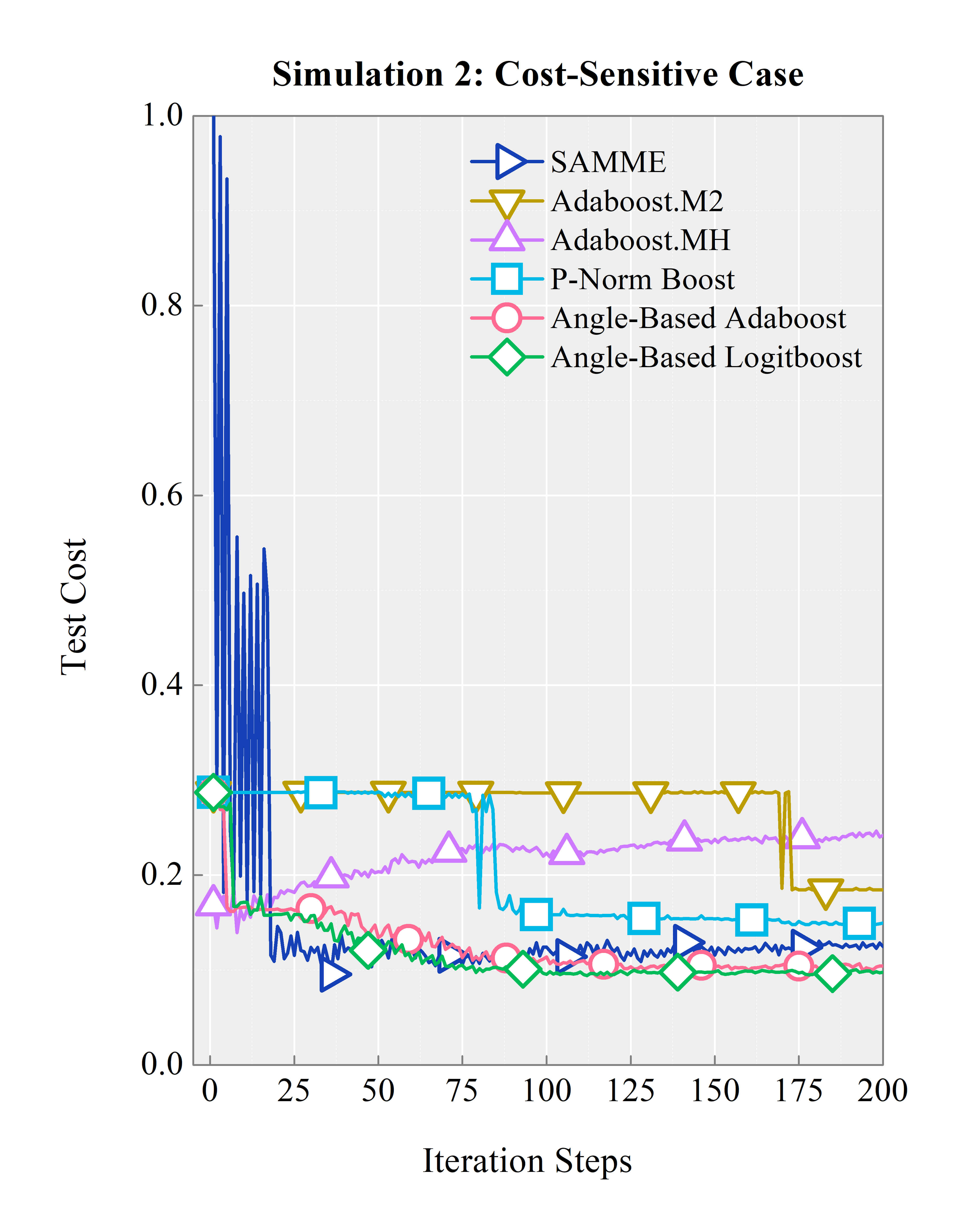}}
		\caption{Test costs of various boosting algorithms as a function of boosting steps in a randomly selected replication in Simulation 2.}
		\label{fig:2}
	\end{figure*}
	
	\subsection{Real Application}
	
	In this subsection, we verify the performance of our angle-based cost-sensitive boosting algorithms via a real-life credit dataset.
	
	In credit rating, borrowers are usually classified into several
	grades to represent their potential ability to pay back the debt and to indicate their risk level of default. Normally, credit rating agencies or financial institutions use letter designations such as A, B, C to represent the credit grade of borrowers or loans. Higher grades are intended to represent a lower probability of default. Classification algorithms are commonly used for credit rating prediction, which greatly support lenders to make more accurate decisions and reduce their loss. Due to the fact that misclassification costs vary across different
	classes in credit rating, multi-class cost-sensitive classifiers are more applicable in this real application.
	
	Hence, this subsection compares the proposed angle-based boosting algorithms and other boosting algorithms on a real credit evaluation dataset. The data used here are loan data from the Lending Club platform. Lending Club is the biggest P2P lending site in the U.S. and their data are publicly available for download \citep{Carlos2015}. We first collect the data of loans on Lending Club from January 2017 to March 2018, which contains $551,448$ observations with $148$ variables. In particular, a variable named \textit{grade} is the credit grade for loans assigned by Lending Club, which is a measure for borrower assessment. Specifically, Lending Club uses the borrower’s FICO credit
	scores along with other information provided in the borrower application to assign a loan credit grade ranging from A to G in descending credit ranks to each loan. After that, Lending Club’s interest rate is derived from the credit rating of the loan plus risk premium, which results in a strong correlation between the interest rate and the assigned loan credit rating \citep{Zhou2018}. Thus in this experiment, the \textit{grade} variable will be used as the class label and we will classify loans into seven credit grades, which is obviously a multi-class classification problem.
	
	The feature selection procedure is then carried out. We first delete some irrelevant features (like loan \textit{id numbe}r and \textit{URL} for the Lending Club page with listing data) as well as the features whose missing values are above 30 \% from the collected data. Some features that are correlated with the loan credit grade are also removed, such as the interest rate and the loan subgrade. This leads to 57 features that are finally preserved. Then, data points with missing values are imputed using a mean/mode replacement for continuous/categorical attributes, respectively. Categorical attributes are also converted by quoting dummy variables. Afterwards, a subsample containing 10,500 observations with 1,500 for each grade has been extracted to form our final processed dataset for the experiment. We randomly select {4\%} instances as the training set, and the remaining is for testing. The standardization procedure is also carried out to ensure each column of continuous attribute has zero mean and unit variance. The classification performance is finally measured by the averaged test cost over 20 independent replications of this credit example.
	
	\subsubsection{Cost-Insensitive Case}
	We first consider the cost-insensitive case, where the equal misclassification costs are applied. Table \ref{CreditRes_table} gives the averaged test costs and their estimated standard errors over $20$ replications of the credit example for all boosting algorithms. Figures \ref{fig:3} displays the test cost curves of all algorithms as functions of boosting steps in a randomly selected replication of credit example with equal misclassification costs. As can be seen from these results, both proposed Angle-Based Adaboost and Angle-Based Logitboost still work well and have very competitive performances as p-norm boosting even though in cost-insensitive case.
	
	\begin{table} 
		\captionsetup{font={footnotesize}}
		\caption{Averaged test cost with their estimated standard errors inside parentheses for all boosting algorithms based on 20 replications of the credit example. Bold values represent the lowest averaged test cost.}
		\label{CreditRes_table}
		\footnotesize
		\centering
		%\resizebox{!}{3.4cm}{
		\begin{tabular}{lcccccc}
			\toprule	
			& {AdaBoost.MH}& {SAMME} & {AdaBoost.M2} & {p-norm Boost}&\thead[c]{Angle-Based \\Adaboost}&\thead[c]{Angle-Based\\ Logitboost}\\
			\midrule
			\multicolumn{7}{l}{Cost-Insensitive Case}\\
			\hline
			Equal Costs  &0.78 & 0.72 & 0.70 & \textbf{0.68} & \textbf{0.68} & \textbf{0.68}  \\
			& (0.0069) & (0.0035) & (0.0032) & \textbf{(0.0015)} & \textbf{(0.0013)} & \textbf{(0.0015)}  \\
			\hline
			\multicolumn{7}{l}{Cost-Sensitive Case}\\
			\hline
			Linear Costs  &2.21 & 1.24 & 1.38 & 1.14 & \textbf{1.10} & \textbf{1.10} \\
			& (0.0716) & (0.0072) & (0.0126) & (0.0043) & \textbf{(0.0039) }& \textbf{(0.0051)}\\
			\hline
			Partitioned- &19.61 & 6.45 & 2.72 & 2.60 & \textbf{2.36} & 2.39 \\
			Linear Costs
			& (1.1272) & (0.0724) & (0.0158) & (0.0080) & \textbf{(0.0113)} & (0.0248) 
			\\
			
			\bottomrule      
		\end{tabular}%}
		\label{bs}
	\end{table}
	
	\begin{figure*}[!h]
		\centering
		\includegraphics[width=200pt]{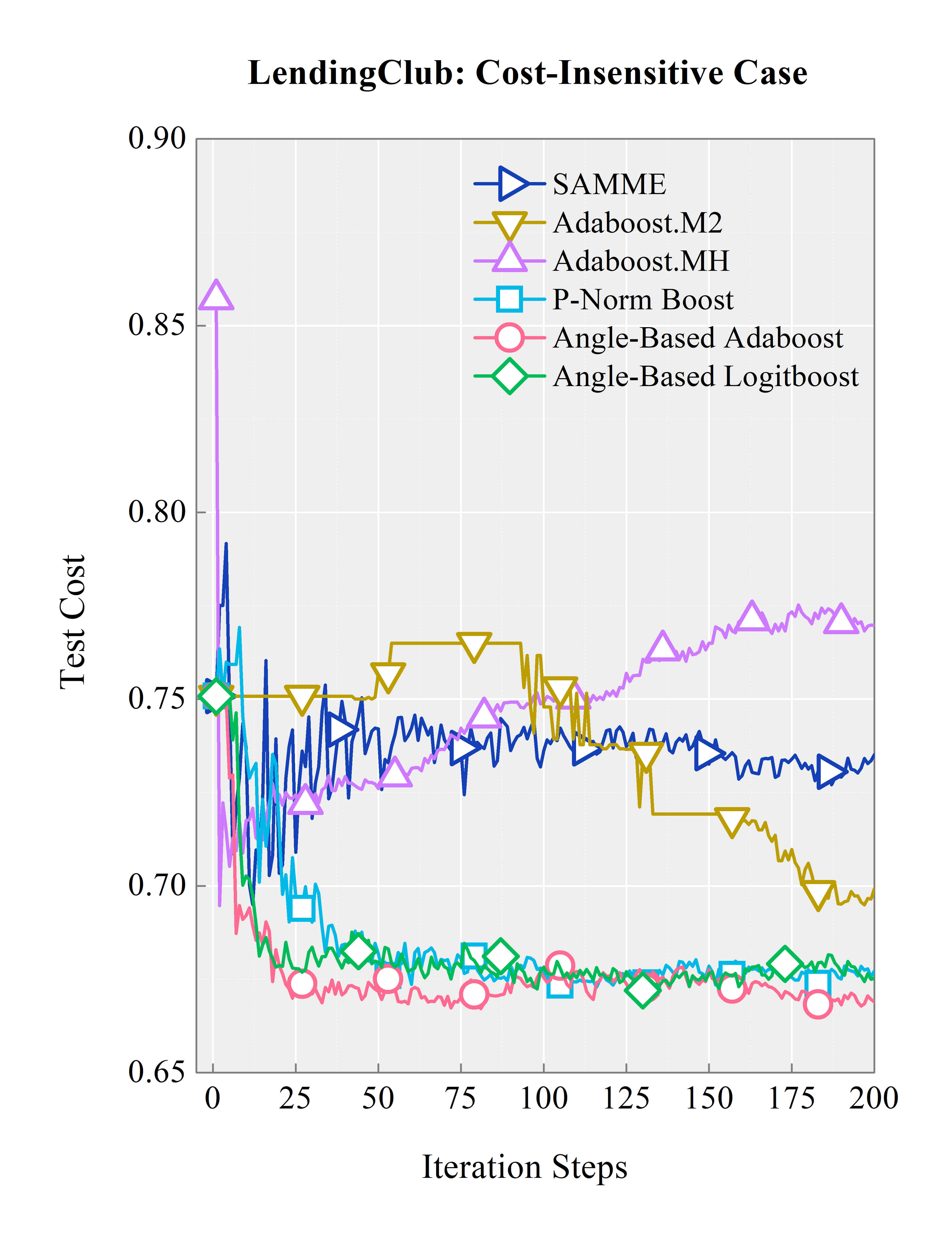}
		\caption{Test costs of various boosting algorithms as a function of boosting steps in a randomly selected replication of credit example in cost-insensitive case.}
		\label{fig:3}
	\end{figure*}

	\subsubsection{Cost-Sensitive Case}
	
	\paragraph{(1) Linear Costs}
	Normally, the costs of misclassification across classes are not uniform in credit rating. For example, the cost resulting from misclassifying a loan of grade C into grade A is larger than one due to misclassifying B into A. Thus, as suggested by \cite{Wang2018}, the linear cost matrix might be more appropriate for credit rating, which is of the form $$C_{j,k}=|j-k|, \quad j, k=1, 2, ..., 7,$$ with A recoded as 1, B as 2 and so on.
	
	The averaged test costs and the corresponding standard errors for various algorithms in this case are also presented in Table \ref{CreditRes_table}. Figure \ref{fig:4} (a) compares the test cost curves of all algorithms with the linear cost matrix. From these results, we can see that the proposed angle-based methods obviously outperform the others when linear cost matrix is applied.

	\begin{figure*}[h]
		\centering
		\subfloat[Linear Costs]{
			\includegraphics[width=200pt]{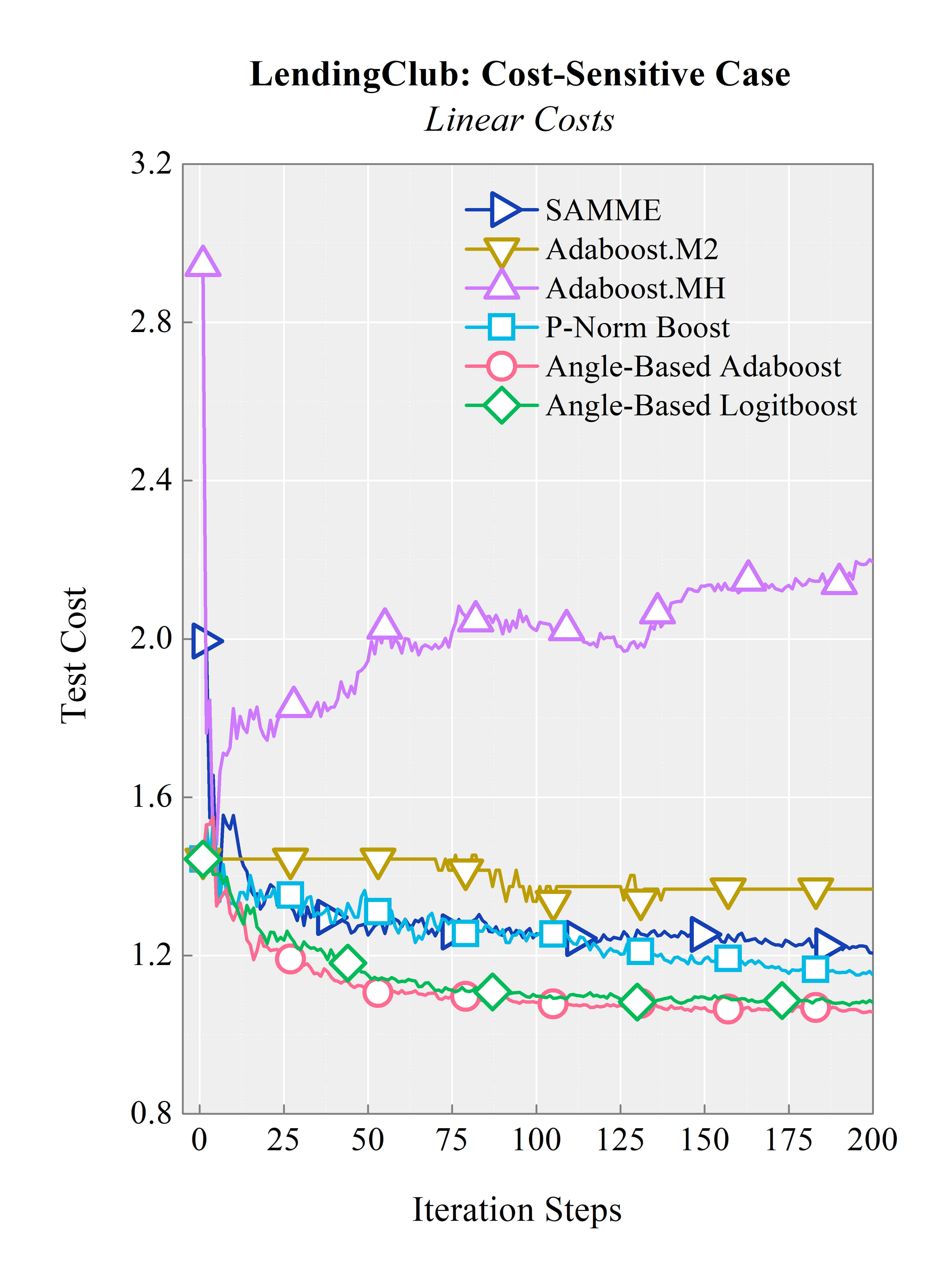}}
		\hspace{4pt}
		\subfloat[Partitioned-Linear Costs]{
			\includegraphics[width=200pt]{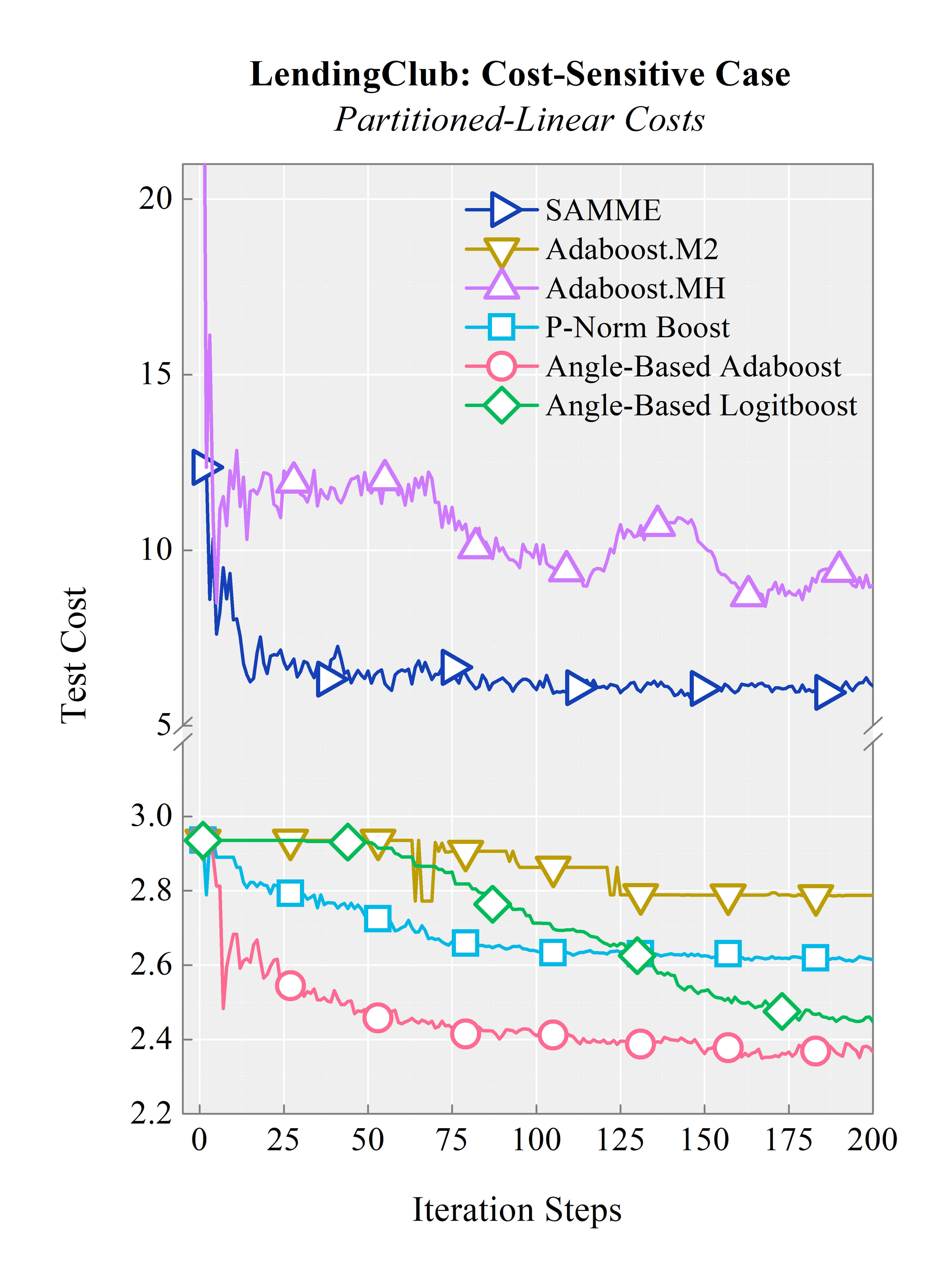}}
		\caption{Test costs of various boosting algorithms as a function of boosting steps in a randomly selected replication of credit example in cost-sensitive case.}
		\label{fig:4}
	\end{figure*}
	
	\paragraph{(2) Partitioned-Linear Costs}
	In fact, the cost of misclassifying bad credit as good is typically much higher than that of misclassifying good credit as bad \citep{Lessmann2015}. As suggested by \cite{Hand2008} and \cite{Wang2018}, the misclassification cost $C_{j,k}$ could be set as ten times $C_{k,j}$ for $j>k$ in credit rating problem. Then the corresponding partitioned-linear cost matrix is defined as 
	\begin{equation}\nonumber
	C_{j,k}=\left\{\begin{array}{ll}{k-j} & {\quad\text{if  } k\geq j,} \\ 
	{10\cdot(j-k)} & {\quad\text{if  } k<j,}
	\end{array}\right.
	\end{equation}
	for $j, k=1, 2, ..., 7.$
	
	The averaged test costs and the standard errors with partitioned-linear costs are also provided in Table \ref{CreditRes_table}. The test cost curves of all algorithms with the partitioned-linear cost matrix in a randomly selected replication are presented in Figure \ref{fig:4} (b). It can be seen from them that the proposed Angle-Based Adaboost works the best overall, with the Angle-Based Loogitboost following behind. This verifies again that the two proposed angle-based algorithms outperform the others in cost-sensitive case for this credit example.

	\section{Conclusion}\label{Sec6}
	\label{sec:conc}
	In this paper, we have proposed a general form of angel-based cost-sensitive multicategory loss function which has many desirable properties such as Fisher Consistency. It could be used to extend large-margin classifiers directly to multicategory versions  in cost-sensitive scenario. Furthermore, since the simplex coding is utilized in our framework, the typical sum-to-zero constraint is removed. Thus, the computational burden of our angle-based methods is reduced. To verify the usefulness of the proposed framework, two novel cost-sensitive multicategory boosting algorithms, namely Angle-Based Adaboost and Angle-Based Logitboost, also have been derived. Numerical experiments conducted on synthetic and real datasets confirm their competitive classification performance compared with other existing boosting algorithms. In future work, the extensions to other loss functions like the large-margin unified loss family \citep{Liu2011} could be carried out. Furthermore, the situation where the cost for correct classification is not equal to zero could also be considered.
	
	\section{Appendix}
	\label{sec:appendix}
	The Lemma 1 of \cite{Zhang2014} is presented as follows.
	\begin{lemma}\label{lemma_zhang}
		Suppose we have an arbitrary $\bs{f}\in\mathbb{R}^{K-1}$. For any $u, v\in\{1,..., K\}$ such that $u\neq v$, define $\bs{T}_{u,v}= \bs{w}_{u}-\bs{w}_{v}$. For any scalar $z\in\mathbb{R}$, $\left< (\bs{f}+z\bs{T}_{u,v}), \bs{w}_{\omega}\right>=\left<\bs{f}, \bs{w}_{\omega}\right>$, where $\omega\in\{1,..., K\}$ and $\omega\neq u, v$. Furthermore, we have that $\left<(\bs{f}+z\bs{T}_{u,v}), \bs{w}_{u}\right>-\left<\bs{f}, \bs{w}_{u}\right>=-\left<(\bs{f}+z\bs{T}_{u,v}), \bs{w}_{v}\right>+\left<\bs{f}, \bs{w}_{v}\right>$.
	\end{lemma}

	%\begin{proof}
	%	Because $\bs{W}$ is a simplex, $\bs{T}_{u,v}\bot\bs{w}_{\omega}$ for any $\omega\neq u\neq v$, and this proves
	%	the first part of Lemma \ref{lemma_zhang}. The second part follows after some calculation.
	%\end{proof}
	
	%\bigskip
	%\begin{center}
	%{\large\bf SUPPLEMENTARY MATERIAL}
	%\end{center}
	%
	%\begin{description}
	%
	%\item[Title:] Brief description. (file type)
	%
	%\item[R-package for  MYNEW routine:] R-package ÒMYNEWÓ containing code to perform the diagnostic methods described in the article. The package also contains all datasets used as examples in the article. (GNU zipped tar file)
	%
	%\item[HIV data set:] Data set used in the illustration of MYNEW method in Section~ 3.2. (.txt file)
	%
	%\end{description}
	
	%\section{BibTeX}
	%
	%We hope you've chosen to use BibTeX!\ If you have, please feel free to use the package natbib with any bibliography style you're comfortable with. The .bst file Chicago was used here, and agsm.bst has been included here for your convenience. 
	
	\bibliographystyle{asa}
	
	\bibliography{AngleBased}

\begin{thebibliography}{44}
\newcommand{\enquote}[1]{``#1''}
\expandafter\ifx\csname natexlab\endcsname\relax\def\natexlab#1{#1}\fi

\bibitem[{Bach et~al.(2006)Bach, Heckerman, and Horvitz}]{Bach2006}
Bach, F.~R., Heckerman, D., and Horvitz, E. (2006), \enquote{Considering cost
  asymmetry in learning classifiers,} \textit{Journal of Machine Learning
  Research}, 7, 1713--1741.

\bibitem[{Bartlett et~al.(2006)Bartlett, Jordan, and McAuliffe}]{Bartlett2006}
Bartlett, P., Jordan, M., and McAuliffe, J. (2006), \enquote{Convexity,
  classification, and risk bounds,} \textit{Journal of the American Statistical
  Association}, 101, 138--156.

\bibitem[{Breiman et~al.(1984)Breiman, Friedman, Olshen, and
  Stone}]{Breiman1984}
Breiman, L., Friedman, J., Olshen, R., and Stone, C. (1984),
  \textit{Classification and Regression Trees}, Belmont, CA: Wadsworth.

\bibitem[{Carlos et~al.(2015)Carlos, Begona, and Luz}]{Carlos2015}
Carlos, S.~C., Begona, G.~N., and Luz, L.~P. (2015), \enquote{Determinants of
  Default in P2P Lending,} \textit{Plos One}, 10, e0139427.

\bibitem[{Drummond and Holte(2000)}]{Drummond2000}
Drummond, C. and Holte, R.~C. (2000), \enquote{Exploiting the Cost
  (In)Sensitivity of Decision Tree Splitting Criteria,} in \textit{Proceedings
  of the Seventeenth International Conference on Machine Learning}, San
  Francisco, CA, USA: Morgan Kaufmann Publishers Inc., ICML '00, pp. 239--246.

\bibitem[{Elkan(2001)}]{Elkan2001}
Elkan, C. (2001), \enquote{The Foundations of Cost-Sensitive Learning,} in
  \textit{In Proceedings of the Seventeenth International Joint Conference on
  Artificial Intelligence}, San Francisco, CA, USA: Morgan Kaufmann Publishers
  Inc., pp. 973--978.

\bibitem[{Fernandez-Baldera and Baumela(2014)}]{FERNANDEZBALDERA2014}
Fernandez-Baldera, A. and Baumela, L. (2014), \enquote{Multi-class Boosting
  with Asymmetric Binary Weak-Learners,} \textit{Pattern Recognition}, 47, 2080
  -- 2090.

\bibitem[{Fernandez-Baldera et~al.(2018)Fernandez-Baldera, Buenaposada, and
  Baumela}]{FERNANDEZBALDERA2018}
Fernandez-Baldera, A., Buenaposada, J.~M., and Baumela, L. (2018),
  \enquote{BAdaCost: Multi-class Boosting with Costs,} \textit{Pattern
  Recognition}, 79, 467 -- 479.

\bibitem[{Freund and Schapire(1997)}]{Freund1997119}
Freund, Y. and Schapire, R.~E. (1997), \enquote{A Decision-Theoretic
  Generalization of On-Line Learning and an Application to Boosting,}
  \textit{Journal of Computer and System Sciences}, 55, 119 -- 139.

\bibitem[{Friedman et~al.(2000)Friedman, Hastie, and Tibshirani}]{Friedman2000}
Friedman, J., Hastie, T., and Tibshirani, R. (2000), \enquote{Special Invited
  Paper. Additive Logistic Regression: A Statistical View of Boosting,}
  \textit{The Annals of Statistics}, 28, 337--374.

\bibitem[{Friedman(2001)}]{Friedman2001}
Friedman, J.~H. (2001), \enquote{Greedy Function Approximation: A Gradient
  Boosting Machine,} \textit{The Annals of Statistics}, 29, 1189--1232.

\bibitem[{Fu et~al.(2018)Fu, Zhang, and Liu}]{Fu2018Adaptively}
Fu, S., Zhang, S., and Liu, Y. (2018), \enquote{Adaptively weighted
  large-margin angle-based classifiers,} \textit{Journal of Multivariate
  Analysis}, 166, 282 -- 299.

\bibitem[{Gu et~al.(2017)Gu, Sheng, Tay, Romano, and Li}]{Gu2017}
Gu, B., Sheng, V.~S., Tay, K.~Y., Romano, W., and Li, S. (2017), \enquote{Cross
  Validation Through Two-Dimensional Solution Surface for Cost-Sensitive SVM,}
  \textit{IEEE Transactions on Pattern Analysis and Machine Intelligence}, 39,
  1103--1121.

\bibitem[{Hand et~al.(2008)Hand, Whitrow, Adams, Juszczak, and
  Weston}]{Hand2008}
Hand, D.~J., Whitrow, C., Adams, N.~M., Juszczak, P., and Weston, D. (2008),
  \enquote{Performance criteria for plastic card fraud detection tools,}
  \textit{Journal of the Operational Research Society}, 59, 956--962.

\bibitem[{Lee et~al.(2004)Lee, Lin, and Wahba}]{Lee2004}
Lee, Y., Lin, Y., and Wahba, G. (2004), \enquote{Multicategory Support Vector
  Machines: Theory and Application to the Classification of Microarray Data and
  Satellite Radiance Data,} \textit{Journal of the American Statistical
  Association}, 99, 67--81.

\bibitem[{Lessmann et~al.(2015)Lessmann, Baesens, Seow, and
  Thomas}]{Lessmann2015}
Lessmann, S., Baesens, B., Seow, H.~V., and Thomas, L.~C. (2015),
  \enquote{Benchmarking state-of-the-art classification algorithms for credit
  scoring: An update of research,} \textit{European Journal of Operational
  Research}, 247, 124--136.

\bibitem[{Lin(2004)}]{Lin2004}
Lin, Y. (2004), \enquote{A note on margin-based loss functions in
  classification,} \textit{Statistics and Probability Letters}, 68, 73--82.

\bibitem[{Liu et~al.(2018)Liu, Liu, Zhu, Initiative, et~al.}]{Liu2018smac}
Liu, L. Y.-F., Liu, Y., Zhu, H., Initiative, A. D.~N., et~al. (2018),
  \enquote{SMAC: Spatial multi-category angle-based classifier for
  high-dimensional neuroimaging data,} \textit{NeuroImage}, 175, 230--245.

\bibitem[{Liu et~al.(2011)Liu, Zhang, and Wu}]{Liu2011}
Liu, Y., Zhang, H.~H., and Wu, Y. (2011), \enquote{Hard or Soft Classification?
  Large-margin Unified Machines,} \textit{Journal of the American Statistical
  Association}, 106, 166--177.

\bibitem[{Lozano and Abe(2008)}]{Lozano2008}
Lozano, A.~C. and Abe, N. (2008), \enquote{Multi-class Cost-sensitive Boosting
  with P-norm Loss Functions,} in \textit{Proceedings of the 14th ACM SIGKDD
  International Conference on Knowledge Discovery and Data Mining}, KDD '08,
  pp. 506--514.

\bibitem[{Mannor et~al.(2002)Mannor, Meir, and Zhang}]{Mannor2002}
Mannor, S., Meir, R., and Zhang, T. (2002), \enquote{The Consistency of Greedy
  Algorithms for Classification,} in \textit{Proceedings of the Annual
  Conference on Computational Learning Theory}, Berlin, Heidelberg, vol. 2375,
  pp. 319--333.

\bibitem[{Masnadi-Shirazi and Vasconcelos(2011)}]{Masnadi-Shirazi2011}
Masnadi-Shirazi, H. and Vasconcelos, N. (2011), \enquote{Cost-Sensitive
  Boosting,} \textit{IEEE Transactions on Pattern Analysis and Machine
  Intelligence}, 33, 294--309.

\bibitem[{Nami and Shajari(2018)}]{Nami2018}
Nami, S. and Shajari, M. (2018), \enquote{Cost-sensitive payment card fraud
  detection based on dynamic random forest and k-nearest neighbors,}
  \textit{Expert Systems with Applications}, 110, 381 -- 392.

\bibitem[{Nocedal and Wright(2006)}]{Nocedal2006}
Nocedal, J. and Wright, S.~J. (2006), \textit{Numerical Optimization (2nd
  ed.)}, Berlin, New York: Springer-Verlag.

\bibitem[{Park et~al.(2011)Park, Chun, and Kim}]{Park2011}
Park, Y.-J., Chun, S.-H., and Kim, B.-C. (2011), \enquote{Cost-sensitive
  Case-based Reasoning Using a Genetic Algorithm: Application to Medical
  Diagnosis,} \textit{Artif. Intell. Med.}, 51, 133--145.

\bibitem[{Sahin et~al.(2013)Sahin, Bulkan, and Duman}]{Sahin2013}
Sahin, Y., Bulkan, S., and Duman, E. (2013), \enquote{A Cost-sensitive Decision
  Tree Approach for Fraud Detection,} \textit{Expert Syst. Appl.}, 40,
  5916--5923.

\bibitem[{Schapire and Singer({1999})}]{Schapire1999}
Schapire, R. and Singer, Y. ({1999}), \enquote{{Improved boosting algorithms
  using confidence-rated predictions},} \textit{{Machine Learning}}, {37},
  {297--336}, {11th Annual Conference on Computational Learning Theory,
  MADISON, WI, JUL 24-26, 1998}.

\bibitem[{Sun et~al.(2007)Sun, Kamel, Wong, and Wang}]{Sun2007}
Sun, Y., Kamel, M.~S., Wong, A. K.~C., and Wang, Y. (2007),
  \enquote{Cost-sensitive boosting for classification of imbalanced data,}
  \textit{Pattern Recognition}, 40, 3358--3378.

\bibitem[{Ting(2000)}]{Ting2000}
Ting, K.~M. (2000), \enquote{A Comparative Study of Cost-Sensitive Boosting
  Algorithms,} in \textit{Proceedings of the Seventeenth International
  Conference on Machine Learning}, San Francisco, CA, USA: Morgan Kaufmann
  Publishers Inc., pp. 983--990.

\bibitem[{Wang et~al.(2018)Wang, Kou, and Peng}]{Wang2018}
Wang, H., Kou, G., and Peng, Y. (2018), \enquote{Cost-sensitive classifiers in
  credit rating: A comparative study on P2P lending,} in \textit{Proceedings of
  7th International Conference on Computers Communications and Control
  (ICCCC)}, pp. 210--213.

\bibitem[{Wang(2013)}]{Wang2013}
Wang, J. (2013), \enquote{Boosting the Generalized Margin in Cost-Sensitive
  Multiclass Classification,} \textit{Journal of Computational and Graphical
  Statistics}, 22, 178--192.

\bibitem[{Yang et~al.(2009)Yang, Wang, Mi, Lin, and Cai}]{Yang2009}
Yang, F., Wang, H.-z., Mi, H., Lin, C.-d., and Cai, W.-w. (2009),
  \enquote{Using random forest for reliable classification and cost-sensitive
  learning for medical diagnosis,} \textit{BMC Bioinformatics}, 10, S22.

\bibitem[{Zadrozny et~al.(2003)Zadrozny, Langford, and Abe}]{Zadrozny2003}
Zadrozny, B., Langford, J., and Abe, N. (2003), \enquote{Cost-Sensitive
  Learning by Cost-Proportionate Example Weighting,} in \textit{Proceedings of
  the 3rd IEEE International Conference on Data Mining}, pp. 435-- 442.

\bibitem[{Zhang and Liu(2013)}]{Zhang2013}
Zhang, C. and Liu, Y. (2013), \enquote{Multicategory Large-margin Unified
  Machines,} \textit{Journal of Machine Learning Research}, 14, 1349--1386.

\bibitem[{Zhang and Liu(2014)}]{Zhang2014}
--- (2014), \enquote{Multicategory Angle-Based Large-Margin Classification,}
  \textit{Biometrika}, 101, 625--640.

\bibitem[{Zhang et~al.(2016{\natexlab{a}})Zhang, Liu, Wang, and
  Zhu}]{Zhang2016Reinforced}
Zhang, C., Liu, Y., Wang, J., and Zhu, H. (2016{\natexlab{a}}),
  \enquote{Reinforced Angle-Based Multicategory Support Vector Machines,}
  \textit{Journal of Computational \& Graphical Statistics}, 25, 806--825.

\bibitem[{Zhang et~al.(2016{\natexlab{b}})Zhang, Sun, Ji, Yuan, and
  Sun}]{Zhang2016F}
Zhang, G., Sun, H., Ji, Z., Yuan, Y.-H., and Sun, Q. (2016{\natexlab{b}}),
  \enquote{Cost-sensitive Dictionary Learning for Face Recognition,}
  \textit{Pattern Recogn.}, 60, 613--629.

\bibitem[{{Zhang} and {Zhou}(2010)}]{Zhang2010F}
{Zhang}, Y. and {Zhou}, Z. (2010), \enquote{Cost-Sensitive Face Recognition,}
  \textit{IEEE Transactions on Pattern Analysis and Machine Intelligence}, 32,
  1758--1769.

\bibitem[{Zhang et~al.(2017)Zhang, Luo, Garca, and Herrera}]{Zhang2017BP}
Zhang, Z.-L., Luo, X.-G., Garca, S., and Herrera, F. (2017),
  \enquote{Cost-Sensitive Back-propagation Neural Networks with Binarization
  Techniques in Addressing Multi-class Problems and Non-competent Classifiers,}
  \textit{Appl. Soft Comput.}, 56, 357--367.

\bibitem[{Zhou et~al.(2018)Zhou, Zhang, and Luo}]{Zhou2018}
Zhou, G., Zhang, Y., and Luo, S. (2018), \enquote{P2P Network Lending, Loss
  Given Default and Credit Risks,} \textit{Sustainability}, 10, 1010.

\bibitem[{Zhou and Liu(2005)}]{Zhou2005Training}
Zhou, Z.~H. and Liu, X.~Y. (2005), \enquote{Training cost-sensitive neural
  networks with methods addressing the class imbalance problem,} \textit{IEEE
  Transactions on Knowledge \& Data Engineering}, 18, 63--77.

\bibitem[{Zhou and Liu(2010)}]{Zhou2010On}
--- (2010), \enquote{On Multi-Class Cost-Sensitive Learning.}
  \textit{Computational Intelligence}, 26, 232--257.

\bibitem[{Zhu et~al.(2009)Zhu, Zou, Rosset, and Hastie}]{Zhu2009}
Zhu, J., Zou, H., Rosset, S., and Hastie, T. (2009), \enquote{Multi-class
  AdaBoost,} \textit{Statistics and Its Interface}, 2, 349--360.

\bibitem[{Zou et~al.(2008)Zou, Zhu, and Hastie}]{Zou2008}
Zou, H., Zhu, J., and Hastie, T. (2008), \enquote{New Multicategory Boosting
  Algorithms Based on Multicategory Fisher-Consistent Losses,} \textit{The
  Annals of Applied Statistics}, 2, 1290--1306.

\end{thebibliography}
\end{document}